\documentclass[twoside]{article}

\usepackage[preprint]{aistats2027}
\usepackage[round]{natbib}

\usepackage{amsmath,amsfonts}
\usepackage{algorithmic}
\usepackage{algorithm}
\usepackage{array}
\usepackage[caption=false,font=normalsize,labelfont=sf,textfont=sf]{subfig}
\usepackage{textcomp}
\usepackage{stfloats}
\usepackage{bbold}
\usepackage{url}
\usepackage{caption}
\usepackage{csquotes}
\usepackage{caption}
\usepackage{subcaption}
\usepackage{verbatim}
\usepackage{graphicx}
\usepackage{cite}
\usepackage{algorithm}
\usepackage{algorithmic}

\usepackage{amsfonts,amsmath,amsthm, mathtools}
\usepackage{amssymb}
\usepackage{caption}
\usepackage{verbatim}
\usepackage{mathrsfs}
\usepackage{algorithm}
\usepackage{algorithmic}
\usepackage{comment}
\usepackage{multirow, tabularx, booktabs}
\usepackage[dvipsnames]{xcolor}

\usepackage{hyperref}
\usepackage{dsfont}

\newtheorem{theorem}{Theorem}[section]

\newtheorem{lemma}[theorem]{Lemma}

\newtheorem{definition}[theorem]{Definition}
\newtheorem{assumption}[theorem]{Assumption}

\usepackage{fullpage}
\usepackage{xspace}

\numberwithin{equation}{section}

\theoremstyle{plain}

\newtheorem*{theorem*}{Theorem}

\theoremstyle{definition}

\newtheorem{prop}{Proposition}[section]

\usepackage[margin=1in]{geometry}

\newcommand{\rr}{\mathbb{R}}

\newcommand{\iin}{{\textrm{in}}}
\newcommand{\pr}{{\textrm{pr}}}
\newcommand{\out}{{\textrm{out}}}

\newcommand{\thg}{\theta_{\tilde{g}}}
\newcommand{\thex}{\theta_{\tilde{a}}}
\newcommand{\prg}{{\textrm{pr,g}}}
\newcommand{\pre}{{\textrm{pr,e}}}

\newcommand{\hd}{\mathcal{H}|_{\Delta,K}}
\newcommand{\hdk}{\mathcal{H}|_{\Delta, K,k}}
\newcommand{\hdp}{\mathcal{H}|_{\Delta, p}}

\newcommand{\xx}{\mathcal{X}}

\begin{document}

%
\runningtitle{Sparsity Saves Compute, Not Samples}

%

\twocolumn[

\aistatstitle{Sparsity Saves Compute, Not Samples: Oracle Inequalities for Dense and Sparse Mixture of Experts}

\aistatsauthor{ Sagar Ghosh \And Huy Nguyen \And  Alessandro Rinaldo }

\aistatsaddress{ The University of Texas at Austin} ]

\begin{abstract}
Mixture-of-Experts (MoE) models combine specialized experts through an input-dependent routing mechanism, and their sparse variants underpin many modern large-scale systems. Yet the statistical cost of adaptive routing remains only partially understood. We study the estimation error of least-squares MoE estimators when both the experts and the gating network are Lipschitz in their parameters. We represent each MoE as a gate--expert pair in a Gated Hilbert Space, whose norm $\|\cdot\|_\Delta$ separates the contributions of routing and expert discrepancies. Using its empirical representation, we derive an oracle inequality with a fast rate of order $\tilde{\mathcal O}\big((Kd_{\mathrm{pr},e}+d_{\mathrm{pr},g})/n\big)$ for dense MoEs with $K$ experts. For sparse Top-$k$ and Top-$p$ routing, the selection map is discontinuous in the gating parameters. We handle this by partitioning the gating class according to the selection patterns it induces on the sample. The resulting bounds match the dense rate up to an additive routing cost governed by the Natarajan dimension of the selector. We show that such a routing-complexity condition cannot be dropped: Lipschitz continuity in the parameters alone does not suffice. For linear softmax gating, the routing cost is of the same order as the gating term. Finally, for linear experts, we establish a matching minimax lower bound of order $Kd_{\mathrm{pr},e}/n$, showing that the rates are optimal up to logarithmic factors.
\end{abstract}

\section{INTRODUCTION}

The Mixture-of-Experts (MoE) framework is a prominent paradigm for modeling heterogeneous input-output relationships. By decomposing complex prediction tasks into specialized subproblems, MoE architectures combine expert functions through a data-dependent gating mechanism that adaptively weights their contributions. Since their introduction by \citep{jacobs1991adaptive} and \citep{jordan1994hierarchical}, MoE models have been widely studied for adaptive function approximation, local modeling, and conditional computation.

MoE architectures have recently regained prominence through sparse systems such as Switch Transformers \citep{fedus2022switch}, GShard \citep{lepikhin2020gshard}, and architectures employing Top-$k$ \citep{shazeer2017outrageously} and Top-$p$ \citep{huang2024harder} routing. These methods activate only a subset of experts per input, enabling substantial capacity gains at manageable computational cost. Despite their empirical success, the statistical theory of adaptive expert selection remains underdeveloped. Existing work has studied approximation \citep{nguyen2016universal}, identifiability \citep{jiang1999identifiability}, consistency \citep{nguyen2018practical}, and parametric mixtures \citep{gormley2019mixture}, while finite-sample guarantees under weak structural assumptions remain limited.

A key theoretical challenge arises from the interaction between expert and gating complexity. While classical learning theory provides tools such as covering numbers, Rademacher complexities, and empirical process methods \citep{wellner2013weak, bartlett2002rademacher, wainwright2019high} for individual function classes, analyzing adaptive mixtures is more subtle because the routing mechanism depends on the input and jointly determines the effective complexity of the predictor.

        
    
    


In this work, we develop a statistical analysis of parametric MoE models under weak regularity assumptions. For dense MoEs, we assume only that the expert and gating functions belong to Lipschitz-parametrized classes, without further architectural constraints, higher-order smoothness, or sparsity. We represent each MoE predictor as a gate--expert pair in a \emph{Gated Hilbert Space}, whose norm $\|\cdot\|_\Delta$ separates the contributions of routing and expert discrepancies, and use it to derive covering-number bounds and a fast-rate oracle inequality. The resulting rate scales with the total parameter count $Kd_\pre+d_\prg$, quantifying the statistical cost of adaptive specialization. Our main contributions are as follows.\\
\noindent\textbf{Functional framework.} We represent MoE predictors as gate--expert pairs in a \emph{Gated Hilbert Space}, whose norm $\|\cdot\|_\Delta$ separates routing and expert discrepancies and yields covering-number bounds for MoE classes.\\
\noindent\textbf{Fast rates for dense MoEs.} Assuming only Lipschitz parametrizations, we show via an offset-Rademacher oracle inequality that the estimation error is $\tilde{\mathcal O}\big((Kd_\pre+d_\prg)/n\big)$.\\
\noindent\textbf{Sparse routing.} For Top-$k$ and Top-$p$ routing, whose selection map is discontinuous in the gating parameters, we partition the gating class by the selection patterns it induces on the sample. \emph{The rates match the dense rate up to a routing cost governed by the Natarajan dimension of the router: sparsity reduces computation, not worst-case sample complexity. We also show that such a routing-complexity condition is necessary.}\\
\noindent\textbf{Lower bounds and simulations.} For linear experts, we prove a minimax lower bound of order $Kd_\pre/n$, so the rates are optimal up to logarithmic factors, and we illustrate the bounds on synthetic data.
\section{RELATED WORKS}
\noindent\textbf{Mixture-of-Experts Models.}
The MoE framework was introduced by \citep{jacobs1991adaptive} and extended to 
hierarchical mixtures by \citep{jordan1994hierarchical}. Early theoretical work 
established identifiability conditions \citep{jiang1999identifiability}, universal 
approximation properties \citep{nguyen2016universal}, and asymptotic consistency 
results \citep{nguyen2018practical}. Comprehensive overviews of MoE modeling 
from statistical perspectives are provided by \citep{gormley2019mixture} and \citep{zeevi1996time}. 

\noindent\textbf{Large-Scale Sparse MoE.}
Modern sparse MoE architectures achieve scalability through conditional 
computation. GShard \citep{lepikhin2020gshard} and Switch Transformers 
\citep{fedus2022switch} demonstrated that sparse routing enables massive capacity 
gains at tractable cost. Subsequent routing refinements include Top-$k$ 
expert-choice routing \citep{zhao2024generalization} and dynamic Top-$p$ routing 
\citep{huang2024harder}, which adapts the number of active experts to input 
complexity.

\noindent\textbf{Statistical Learning Theory.}
Our finite-sample analysis builds on classical tools from nonparametric 
regression \citep{gyor}, VC theory and covering numbers \citep{huss, anthony, bartlett2002rademacher}, and 
Rademacher complexity \citep{bartlett2002rademacher}. The empirical process 
methods employed throughout follow \citep{wellner2013weak} \citep{wainwright2019high}, and \citep{nguyen2023demystifying,nguyen2024statistical,nguyen2024sigmoid,nguyen2026convergence}.  Closest to our setting, \citep{nguyen2026convergence} study least-squares estimation of a softmax-gated MoE regression function and obtain a parametric rate for the regression function, with emphasis on expert estimation under identifiability conditions, and \citep{nguyen2024statistical} study top-$K$ sparse softmax gating in a Gaussian MoE. \citet{zhao2024generalization} bound the generalization gap of top-$k$ MoE at rate $O(1/\sqrt n)$ through the Natarajan dimension of the router. Our results differ in three respects. (i) We prove an oracle inequality under squared loss that allows misspecification, for general Lipschitz-parametrized experts and gates, without tying the analysis to softmax gating or to identifiability conditions. (ii) We make the dependence on the number of experts explicit, $Kd_{\mathrm{pr},e}$, for dense, Top-$k$ and Top-$p$ routing. (iii) We prove a matching minimax lower bound showing that sparse routing does not reduce this dependence, and show that a routing-complexity condition cannot be dropped.

\section{BACKGROUND}

\noindent

In this section, we introduce a functional framework, which we call \emph{Gated Hilbert Space} that accommodates both dense and sparse Mixture-of-Experts (MoE) architectures. We represent each MoE predictor as an ordered pair of a gating tuple and an expert tuple, and equip the resulting class with the metric $d_\Delta$, induced by embedding $(\tilde g,\tilde a)\mapsto(g_1a_1,\ldots,g_Ka_K)$ into $L^2(\mu;\mathbb R^{d_\out})^K$. This metric separates the contributions of routing and expert discrepancies, which we use to derive covering-number bounds for the dense class; for the sparse classes, we additionally control the discontinuity of the routing map through selection patterns.



\textbf{Notation.} The input space $\xx\subseteq\rr^{d_\iin}$ is pre-compact, equipped with a $\sigma$-finite Borel probability measure $\mu$. Each function in the MoE class is assumed to be a parametric $L$-Lipschitz function (See Appendix~\ref{defn:parametric_lipschitz} for a definition), where each function $\tilde{g}$ in the Gating Class $\mathcal{G}$ is parametrized with $\tilde{g} = \tilde{g}(\cdot;\thg)$ where $\thg\in \Theta_\mathcal{G}$, and $\Theta_\mathcal{G}\subset \rr^{d_\prg} $is pre-compact. Additionally, each function $a$ in the Expert Class $\mathcal{E}$ is parametrized with $a = (\cdot;\theta_a)$ where $\theta_a\in\Theta_\mathcal{E}$ and $\Theta_\mathcal{E} \subset\rr^{d_\pre} $is also pre-compact. Each MoE function $\tilde{f} = (\tilde{g}, \tilde{a})$  consisting of the gating tuple and the expert tuple is then parametrized by $\theta_f = (\thg, \thex)\in \Theta_\mathcal{G}\times \Theta_\mathcal{E}^K = \Theta_\pr\subset \rr^{d_\prg + K d_\pre}=\rr^{\pr}$ and $\Theta_{\pr}$ is also pre-compact. To establish a rigorous functional framework for Mixture-of-Experts (MoE) models, we introduce a structured function space that jointly represents the gating and expert components of an MoE architecture. Throughout, we regard each MoE predictor as an ordered pair consisting of a gating tuple and an expert tuple. We denote by $\hd$ the class of $K$-expert models:

\begin{equation}\label{eqn:defn_general_moe}
    \scalebox{0.88}{$
    \begin{aligned}
    \hd \coloneqq \Bigl\{ \tilde{f} = (\tilde{g}, \tilde{a}) = (g_1,\ldots,g_K,a_1,\ldots,a_K) \;\Big|\; \\
    g_i:\mathbb{R}^{d_{\mathrm{in}}}\to[0,1],\quad
    \textstyle\sum_{i=1}^K g_i(x)=1,\quad
    a_i:\mathbb{R}^{d_{\mathrm{in}}}\to\mathbb{R}^{d_{\mathrm{out}}}
    \Bigr\}
\end{aligned}$}
\end{equation}

The functions ($g_i$) represent the gating weights, while the functions ($a_i$) correspond to the outputs of the ($K$) experts. Note that, the function $\tilde{f}$ depends on $\theta_f\in \Theta_\pr$, to keep the notations coherent, we will only write $\tilde{f}(x)$ instead of $\tilde{f}(x;\theta_f)$ for any $x\in\xx$. Thus, every element of ($\hd$) uniquely specifies an MoE model through its routing and expert components. Additionally, $\hdk\subseteq\hd$ and $\hdp\subseteq\hd$ denote the subclasses with only top-$k$ and top-$p$ active experts, respectively.  We consider two sparse routing strategies for Mixture-of-Experts (MoE) models, Top-$k$ and Top-$p$ routing. Given routing probabilities $P={g_1,\ldots,g_K}$ with $\sum_{i=1}^K g_i(x)=1$, Top-$k$ routing selects the $k$ experts with the largest probabilities and renormalizes their weights, yielding  $f(x)=\sum_{l=1}^kg^\prime_{i_l}(x)a_{i_l}(x)$, where $\tilde{g}=\{g_1, ..., g_K\}$ [which add up to 1] and $g_{i_1}, ..., g_{i_k}$ are the top-k probabilities corresponding to the input $x$ subject to the MoE function $f = (\tilde{g}, \tilde{a})\in\hdk$, and $g^\prime_{i_l}=\frac{g_{i_l}}{\sum_{j=1}^kg_{i_j}}$. In contrast, Top-$p$ routing dynamically selects the smallest number of experts whose cumulative routing probability exceeds a prescribed threshold $p$, i.e., $k_x=\min\{j:\sum_{i=1}^j g_{(i)}(x)\geq p\}$, where $g_{(1)}(x)\geq\cdots\geq g_{(K)}(x)$. Thus, Top-$k$ enforces a fixed number of active experts, whereas Top-$p$ adapts the number of active experts to the routing confidence of each input.


For two elements $\tilde{f}=(\tilde{g},\tilde{a}),
\text{ and } \tilde{r}=(\tilde{h},\tilde{b}) \in\hd$, we define the bilinear form
\begin{equation}\label{eqn:gated_inner_product}
\scalebox{0.88}{$
\begin{aligned}
\langle \tilde{f},\tilde{r}\rangle_{\Delta}
&:= \int_{\mathcal X} \sum_{i=1}^{K} \big\langle g_i(x)a_i(x), h_i(x)b_i(x) \big\rangle d\mu(x)\\
&= \int_{\mathcal X} \sum_{i=1}^{K} g_i(x)h_i(x) \langle a_i(x),b_i(x)\rangle d\mu(x),
\end{aligned}$}
\end{equation}

where ($\langle \cdot,\cdot\rangle$) denotes the standard inner product on $\rr^{d_\out}$. 

The bilinearity and symmetry of ($\langle\cdot,\cdot\rangle_{\Delta}$) follow immediately from the corresponding properties of the Euclidean inner product. Moreover, the norm induced by \eqref{eqn:gated_inner_product} is given by
\begin{equation}\label{eqn:delta_norm}
    \scalebox{0.88}{$
    \begin{aligned}
        \|\tilde{f}\|_\Delta\coloneqq \langle \tilde{f}, \tilde{f}\rangle_\Delta^{1/2} = \left[\int_\xx \sum_{i=1}^K g_i(x)^2\|a_i(x)\|^2d\mu(x)\right]^{1/2}\geq 0. 
    \end{aligned}$}
\end{equation}

with equality if and only if $g_i(x)a_i(x)=0 \quad \text{for }\mu\text{-almost every }x$, and for every ($i=1,\ldots,K$). Consequently, after identifying elements that agree almost everywhere, ($\hd$) becomes an inner-product space.

This norm naturally measures the joint contribution of the gating and expert functions, weighting each expert output according to its corresponding routing coefficient. The associated metric on ($\hd$) is therefore

\begin{equation}\label{eqn:distance}
    \scalebox{0.88}{$
    \begin{aligned}
        d_\Delta(f,r) = \|\tilde{f}-\tilde{r}\|_\Delta
    = \left[\int_\xx \sum_{i=1}^K \|g_i(x)a_i(x)-h_i(x)b_i(x)\|^2\right]^{1/2}. 
    \end{aligned}$}
\end{equation}

The metric ($d_\Delta$) quantifies discrepancies between two MoE models through the combined effect of differences in both their routing mechanisms and expert outputs. In fact, the space $(\hd, \langle \cdot,\cdot\rangle_\Delta)$ is complete under the definition of this metric [See Proposition~\ref{prop:completeness} in Appendix~\ref{supp:Background}]. This geometric structure will serve as the foundation for the metric entropy and covering number analysis developed in the subsequent sections.

\section{ESTIMATION ERROR RATES }

In this section, we derive the squared estimation error rates for the general MoE class, as well as for its sparse Top-$k$ and Top-$p$ variants. Specifically, Section~\ref{subsection:general_moe_class} presents the estimation error rates for the general MoE class, while Sections~\ref{subsection:top-k_moe_class} and~\ref{subsection:top-p_moe_class} establish the corresponding rates for the Top-$k$ and Top-$p$ MoE classes, respectively. Finally, Section~\ref{subsection:minimax_moe_class} derives lower bound on the squared estimation error rate and demonstrates that the obtained rates are minimax optimal up to a logarithmic factor, provided the class of experts contain the linear experts. 

\subsection{General MoE Class} 
\label{subsection:general_moe_class}

We begin with a covering-number bound for the dense class $\hd$ of Eq.~\eqref{eqn:defn_general_moe} in the gated metric $d_\Delta$. Inserting $g_ib_i$ in \eqref{eqn:delta_norm} splits $d_\Delta(f,r)$ into an expert discrepancy, weighted by the gates, and a routing discrepancy, weighted by the expert magnitudes. Let $G_\ast =\sup_{\tilde{g}\in\mathcal G}\left(\int_\xx\sum_{i=1}^Kg_i^2(x)d\mu(x)\right)^{1/2}$ be the uniform gatng envelope over the gating class $\mathcal G$ and $B_\infty = \sup_{\tilde{a}\in\mathcal{E}}{\max_{1\leq i\leq K}}{\sup_{x\in\xx}}\|a_i(x)\|$ be the uniform expert bound,  covering the two parametric families separately gives (Proposition~\ref{prop:general_entropy_estimate}, proof in Appendix~\ref{proof:general_entropy_estimate})
\begin{equation}
    \scalebox{0.8}{$
    \begin{aligned}
        \mathcal N(\epsilon,\hd,d_\Delta)\lesssim{}&
\Big(\frac{2\sqrt{2}LG_\ast \sqrt{d_\out}}{\epsilon}\Big)^{d_\pre K} \times\Big(\frac{2\sqrt{2}L \sqrt{K}\,B_\infty}{\epsilon}\Big)^{d_\prg}. 
    \end{aligned}$}
\end{equation}

The bound holds for any probability measure $\mu$ on $\xx$, including the empirical measure. We use it's empirical analouge with the offset Rademacher complexity framework of \citet{rakhlin2020mathematical}, which requires exactly this empirical cover.

\begin{assumption}[Data model]\label{assump:data}
We observe $(x_t,y_t)_{t=1}^{n_{\textrm{sample}}}$ with $y_t=f_\rho(x_t)+\xi_t$. The design points $x_t\in\xx$ are fixed (or we condition on them), and the noise vectors $\xi_t\stackrel{\textrm{i.i.d.}}{\sim}\mathcal N(0, \sigma^2 I_{d_\out})$ i.e.\ $\mathbb E\,e^{\langle u,\xi_t\rangle}\le e^{\sigma^2\|u\|^2/2}$ for all $u$. The experts are bounded by $B_\infty$, so every $\tilde{f}\in\hd$, being a convex combination of expert outputs, satisfies $\|\tilde{f}(x)\|\le B_\infty$.
\end{assumption}
Let $\|\tilde{h}\|_n^2:=\frac1n\sum_{t=1}^n\|\tilde{h}(x_t)\|^2$, let $\hat f$ be the least-squares estimator over $\hd$, fitted using $n (=n_\textrm{train} < n_\textrm{sample})$ samples, and define the offset complexity
\begin{equation}\label{eqn:offset_def}
    \scalebox{0.88}{$
    \begin{aligned}
        \phi(\hd,n):=\mathbb E\sup_{\tilde{h}\in\hd-\hd}\bigg\{&\frac{2}{n}\sum_{t=1}^n\langle\xi_t,\tilde{h}(x_t)\rangle -\frac14\|h\|_n^2\bigg\}.
    \end{aligned}$}
\end{equation}

\begin{lemma}[Oracle inequality]\label{lem:oracle}
Under Assumption~\ref{assump:data},
\begin{align}
\mathbb E\big[\|\hat f-f_\rho\|_n^2\big]\le{}&
3\underbrace{\inf_{\tilde f\in\hd}\|\tilde f-f_\rho\|_n^2}_{\textnormal{Approximation Error}} \nonumber\\
&+2\underbrace{\phi(\hd,n)}_{\textnormal{Estimation Error}}.
\label{eq:oracle_ineq}
\end{align}
If $f_\rho\in\hd$ (well-specified case), then $\mathbb E\|\hat f-f_\rho\|_n^2\le\tfrac43\,\phi(\hd,n)$.
\end{lemma}

Because $\hd$ is not convex, the constant in front of the approximation error cannot in general be taken equal to one \citep{liang2015learning}; the factor $3$ is the price of misspecification. The approximation error depends on the unknown $f_\rho$, so we focus on the estimation error $\phi(\hd,n)$, which quantifies how accurately $\hat f$ recovers the best-in-class predictor.

To bound $\phi(\hd,n)$, write $\xi_t=\sigma z_t$ with $z_t\stackrel{\mathrm{i.i.d.}}{\sim}\mathcal N(0,I_{d_\out})$. Then
\begin{align*}
\phi(\hd,n)=2\sigma\,\mathbb E\sup_{h\in\hd-\hd}\bigg\{&\frac1n\sum_{t=1}^n\langle z_t,h(x_t)\rangle\\
&-\frac{1}{8\sigma}\|h\|_n^2\bigg\},
\end{align*}
that is, $2\sigma$ times the offset Gaussian complexity of $\hd-\hd$ with offset constant $c=1/(8\sigma)$. Writing $\mathcal N_n(\epsilon):=\mathcal N(\epsilon,\hd-\hd,L^2(P_n))$, \citet[Theorem~7]{rakhlin2020mathematical} gives
\begin{align}
\phi(\hd,n)\le 2\sigma&\inf_{0\le\alpha\le\gamma}\Big\{\frac{\log\mathcal N_n(\gamma)}{cn}+4\alpha \nonumber\\
&+\frac{12}{\sqrt n}\int_\alpha^\gamma\!\sqrt{\log\mathcal N_n(\epsilon)}\,d\epsilon\Big\}, \label{eqn:offset_rad}
\end{align}
where $\mathcal N_n(\gamma)\le\mathcal N(\gamma/2,\hd,L^2(P_n))^2$. For parametric classes, the choice $\alpha=\gamma$ already gives the optimal rate up to logarithmic factors: it removes the entropy integral, leaving a single-scale bound that we optimize over $\gamma$. Throughout, $\lesssim$ hides multiplicative constants depending only on $\sigma$, $B_\infty$ and $d_\out$, and Lemma~\ref{lem:oracle} and Eq.\eqref{eqn:offset_rad} apply verbatim to $\hdk$ and $\hdp$. This yields the following bound, expressed in terms of the Lipschitz constant $L$ and the expert bound $B_\infty$ (proof in Appendix~\ref{proof:complexity_general}).

\begin{theorem}\label{thm:complexity_general}
Suppose the gating and expert classes are parametric $L$-Lipschitz and the experts are uniformly bounded, $B_\infty$ and $G_\ast$ defined as before. Then under Assumption~\ref{assump:data}, we have

\begin{equation}
\scalebox{0.88}{$
\begin{aligned}
    \phi(\hd,n)\lesssim &~\frac{D}{n}+\frac{d_\pre K\log\!\big(L\sqrt{d_\out}\,n/D\big)}{n} \\
&+\frac{d_\prg\log\!\big(L K B_\infty\,n/D\big)}{n},
\end{aligned}$}
\label{eqn:complexity_general}
 \end{equation}
where $D:=d_\pre K+d_\prg$.
\end{theorem}

\noindent 
\textbf{Remark.}
    The bound decomposes the squared estimation error into an effective parametric term and logarithmic complexity terms. The first term, $\frac{d_\pre K + d_\prg}{n},$ captures the leading-order contribution associated with the effective number of parameters across the $K$ experts and the gate. The second and third terms account for the additional complexity incurred by controlling the gating and expert function classes, respectively, and enter only through logarithmic factors in the sample size $n$ and the envelope and Lipschitz constants $G_\ast$ and $B_\ast$. Thus, up to logarithmic factors, the squared estimation error scales as $\phi(\hd,n)
\lesssim \frac{d_\pre K+ d_\prg }{n}.$ Therefore upto logarithmic factors, the bound exhibits the linear dependence on the number of experts $K$, reflecting the complexity of allowing all $K$ experts to participate in the general mixture-of-experts architecture.

\subsection{Top-k MoE Class}
\label{subsection:top-k_moe_class}

Top-$k$ routing raises two issues absent in the dense case. First, the selection map $x\mapsto\sigma_{\tilde g}(x)$ is discontinuous in the gating parameters, so nearby gates can activate different experts near ties. Second, different inputs activate different $k$-subsets, so all $Kd_\pre$ expert parameters must be controlled. Following (Zhao, 2024), we condition on the sample, partition the gating class by the selection pattern it induces, bound the number of patterns via the Natarajan dimension $d_N$, and cover each piece, on which the model is Lipschitz. The resulting rate (Theorem~\ref{thm:complexity_top_k}, with proof in Appendix~\ref{proof:complexity_top_k}) matches the dense rate up to a routing cost $d_N(\log n + k\log(eK/k))/n$, which is absorbed into the gating term for linear softmax gates.
\begin{assumption}\label{assump:topk}
\emph{(Bounded experts and routing complexity)}
\textbf{(i)} The experts are uniformly bounded:
    $B_\infty := \sup_{\tilde a\in\mathcal E}\max_{1\le i\le K}\sup_{x\in\xx}\|a_i(x)\| < \infty$.\\
\noindent 
\textbf{(ii)} a fixed measurable rule, and the induced selector class breaks ties in the top-$k$ selections
    \begin{equation}
        \scalebox{0.8}{$
        \begin{aligned}
            \mathcal S_k := \Big\{\, x\mapsto \sigma_{\tilde g}(x) := \operatorname*{arg\,top\text{-}k}\big(g_1(x),\ldots,g_K(x)\big) \;:\; \tilde g\in\mathcal G \Big\}\nonumber
        \end{aligned}$}
    \end{equation}
of maps $\xx\to\binom{[K]}{k}$ has finite Natarajan dimension $d_N := \operatorname{Ndim}(\mathcal S_k) < \infty $ \citep{zhou2022mixture}.
\end{assumption} 

\begin{definition}[Soft-then-hard Top-$k$ routing]
\label{def:class1}
Let $K,k\in\mathbb{N}$ with $1\le k\le K$. Define
{
\begin{align*}
\mathcal{F}_1 \coloneqq \Bigl\{&\tilde{f} = 
(\hat{g}_1,\ldots,\hat{g}_K,a_1,\ldots,a_K) \;\Big|\;\\
&\; g_i:\mathcal{X}\to[0,1] \text{ measurable},\; \textstyle\sum_{i=1}^K g_i(x) = 1,\\
&\; \sigma_g(x) \coloneqq \operatorname{arg\,top\text{-}k}
   \big(g_1(x),\ldots,g_K(x)\big),\\
&\; \hat{g}_i(x) \coloneqq \tfrac{g_i(x)\,
   \mathbb{1}[i \in \sigma_g(x)]}
   {\sum_{j\in\sigma_g(x)} g_j(x)}
\Bigr\}
\end{align*}
}
\end{definition}
 
\begin{theorem}\label{thm:complexity_top_k}
Let $\hdk$ be the parametric soft-then-hard top-$k$ class $\mathcal F_1$ of Definition~\ref{def:class1}, whose gating and expert classes are parametric $L$-Lipschitz, and suppose Assumptions~\ref{assump:topk}\&~\ref{assump:data} holds. Then, for every $1\le k\le K$ [with $D = d_\pre K+d_\prg$],
\begin{equation}\label{eqn:complexity_top-k}
    \scalebox{0.88}{$
    \begin{aligned}
    &\phi(\hdk, n) \lesssim\;\frac{D}{n}+\frac{d_N\big(\log n+2k\log(eK/k)\big)}{n}\\
    &+\frac{d_\pre K\log\!\big(\frac{L\sqrt{d_\out}\,n}{D}\big)+d_\prg\log\!\big(\frac{LKB_\infty\,n}{D}\big)}{n}
    .
\end{aligned}$}
\end{equation}

\end{theorem}

\citep{zhao2024generalization} bound the generalization gap of top-$k$ MoEs for bounded Lipschitz losses, at rate $O(1/\sqrt{n})$, through the Rademacher complexity of a single expert class and the Natarajan dimension $d_N$ of the router. We use the same pattern-counting device for the routing cost, but our setting and conclusions differ in four ways. (i) We study the estimation error under squared loss and obtain fast $\tilde{\mathcal O}(1/n)$ rates via localization. (ii) The expert contribution scales as $Kd_\pre$, and a matching lower bound shows this is unavoidable, because input-dependent routing activates every expert somewhere in the input space. (iii) For linear gating, the routing cost reduces to $d_\prg\log(nK^2)$, without the $k\log(K/k)$ factor. (iv) We also showed that a finite routing complexity is necessary: Lipschitz continuity in the parameters alone does not suffice.\\
 \noindent
 \textbf{Remark (The Routing Term).}
The first line of \eqref{eqn:complexity_top-k} matches the dense rate of Theorem~\ref{thm:complexity_general}: all $K$ experts enter, because different inputs activate different $k$-subsets, so every expert is active somewhere in $\xx$. Sparse routing does not reduce the worst-case parameter count; it adds the routing cost in the second term because the top-$k$ map is discontinuous in the gating parameters. The term $\log\binom{K}{k}\le k\log(eK/k)$ enters additively through the number of selection patterns. \\
 \noindent
 \textbf{Remark (Assumption~\ref{assump:topk}(ii) cannot be dropped). }
Lipschitz continuity in the parameters alone does not suffice for top-$k$ routing. Let $\xx=[0,1]$, $\mu=\mathrm{Unif}[0,1]$, $K=2$, $k=1$, $\Theta_G=(0,1]$, $g_1(x;\theta)=\tfrac12+\tfrac14\theta^2\sin(x/\theta)$, $g_2=1-g_1$, and constant experts $a(x;\vartheta)=\vartheta\in[-1,1]$. Since $|\partial_\theta g_1|\le\tfrac34$, the gate is Lipschitz, yet the top-$1$ output with $a_1\equiv1$, $a_2\equiv-1$ is $\operatorname{sgn}\sin(x/\theta)$. For $\theta_j=1/(2^j\pi)$ this is the $j$-th Rademacher function $r_j$, and $\|r_j-r_l\|_{L^2(\mu)}^2=2$ for $j\neq l$. Under Gaussian noise, every pairwise KL divergence is at most $n/\sigma^2$ while the number of hypotheses is unbounded, so Fano's inequality \citep{tsybakov2008lower} gives a minimax risk of at least $\tfrac12$ for every $n$. With the same gate, the dense output $\tfrac12\theta^2\sin(x/\theta)$ is Lipschitz in $\theta$, as it is covered by Theorem~\ref{thm:complexity_general}.\\
\noindent 
\textbf{Remark (Linear softmax gating).}
For logits $s_i(x)=w_i^\top x+c_i$ ($d_\prg=K(d_{\mathrm{in}}+1)$), $\sigma_{\tilde g}(x)$ depends only on the signs of $s_i(x)-s_j(x)$, $i<j$, which are linear in $\theta_g$. By (Lemma 17; \citep{bartlett2019nearly}), for $n\ge d_\prg$ the number of selection patterns on the sample is at most $2\big(enK(K-1)/d_\prg\big)^{d_\prg}$, so the routing term in \eqref{eqn:complexity_top-k} can be replaced by $d_\prg\log(enK^2/d_\prg)/n$, which is of the same order as the gating term. Hence $\phi(\hdk,n)=\tilde{\mathcal O}\big((d_\pre K+d_\prg)/n\big)$.

\subsection{Top-p MoE Class}
\label{subsection:top-p_moe_class}

Unlike top-$k$, top-$p$ routing activates an input-dependent number of experts. Writing $g_{(K)}(x)\ge\cdots\ge g_{(1)}(x)$ for the sorted gating weights, with ties broken by a fixed rule, the active set at $x$ is
\begin{equation}
    \scalebox{0.8}{$
    \begin{aligned}
        S_{\tilde g}(x) := \{(1),\ldots,(k_x)\},\quad k_x := \min\Big\{j : \textstyle\sum_{l=1}^{j} g_{(K-l+1)}(x)\ge p\Big\}.
    \end{aligned}$}
\end{equation}

Since the $j$ largest of $K$ weights carry mass at least $j/K$, every active set has size at most $\kappa:=\lceil Kp\rceil$. The analysis mirrors the top-$k$ case, with three modifications: the selector ranges over all subsets of size at most $\kappa$, of which there are at most $(eK/\kappa)^\kappa$; the selected mass is bounded below by $p$ rather than $k/K$; and routing complexity is measured by the Natarajan dimension of the top-$p$ selector class.
\begin{assumption}\label{assump:topp}
    \emph{(Top-$p$ routing complexity)}
The selector class $\mathcal S_p:=\{x\mapsto S_{\tilde g}(x):\tilde g\in\mathcal G\}$ has finite Natarajan dimension $d_{N,p}<\infty$.
\end{assumption}

\begin{theorem}\label{thm:complexity_top-p}
Let $\hdp$ be the top-$p$ MoE class with parametric $L$-Lipschitz gating and expert classes, and suppose Assumptions~\ref{assump:topk}(i) and~\ref{assump:topp} hold. Then, under the model assumption in ~\ref{assump:data}, for every $p\in(0,1)$ [$D=d_\pre K+d_\prg, \kappa = \lceil Kp\rceil$],
\begin{equation}\label{eqn:complexity_top-p}
    \scalebox{0.88}{$
    \begin{aligned}
    &\phi(\hdp, n) \lesssim \frac{D}{n} + \frac{d_{N,p}\big(\log n+2\kappa\log(eK/\kappa)\big)}{n}\\
    & \frac{d_\pre K \log\left(\frac{L\sqrt{d_\out}n}{D}\right)}{n} + 
    \frac{d_\prg\log\left( \frac{L\kappa B_\infty n}{Dp}\right)}{n}.
    \end{aligned}$}
\end{equation}
\end{theorem}

\noindent
\textbf{Remark (Interpretation).}
As in the top-$k$ case, all $K$ experts enter the bound, because different inputs activate different subsets. The threshold $p$ affects the rate only through the routing term, via $\kappa=\lceil Kp\rceil$, and through the factor $\kappa/p\le K+1/p$ inside the gating logarithm. Taking $p=k/K$ gives $\kappa=k$ and $\kappa/p=K$, so \eqref{eqn:complexity_top-p} has the same form as \eqref{eqn:complexity_top-k}; the only differences are the selector class and the variable size of the active set. Assumption~\ref{assump:topp} cannot be dropped either: for $K=2$ and $p\le\tfrac12$, top-$p$ routing reduces to top-$1$, and the counterexample following Theorem~\ref{thm:complexity_top_k} applies verbatim.

\subsection{Minimax Optimality of Squared Estimation Error Rates for MoE Class}
\label{subsection:minimax_moe_class}

We show that the factor $Kd_\pre$ in Theorems~\ref{thm:complexity_general}\&\ref{thm:complexity_top-p} cannot be improved, even for sparse routing. The hard instance routes disjoint groups of design points to different experts, so that estimation decouples into $K$ separate linear regressions.
 
\begin{assumption}[Hard instance]\label{assump:lower}
Let $d:=d_{\mathrm{in}}d_\out$ with $Kd\ge8$, and let $n=Km$ with $m\ge d_{\mathrm{in}}$.
\begin{enumerate}
    \item[(i)] \emph{Design.} The design points are split into $K$ groups $\{x_{j,s}\}_{s=1}^m$, $j\in[K]$, of distinct points with $\|x_{j,s}\|\le1$ and $\frac{\lambda}{d_{\mathrm{in}}}I\preceq\Sigma_j:=\frac1m\sum_{s=1}^m x_{j,s}x_{j,s}^\top\preceq\frac{1}{d_{\mathrm{in}}}I$ for some $\lambda\in(0,1]$.
    \item[(ii)] \emph{Routing.} $\mathcal G$ contains a gate $g^0$ with $g^0_j(x_{j,s})=1$ for all $j,s$.
    \item[(iii)] \emph{Experts.} $\mathcal E$ contains the linear maps $x\mapsto Wx$ with $W\in\mathbb R^{d_\out\times d_{\mathrm{in}}}$ and $\|W\|_F\le R$, where $R\le B_\infty$.
\end{enumerate}
\end{assumption}
For example, $x_{j,s}=c_je_{l(s)}$ with distinct $c_j\in[\tfrac12,1]$ and $l(s)$ cycling through $[d_{\mathrm{in}}]$ satisfies (i) with $\lambda=\tfrac14$ whenever $d_{\mathrm{in}}\mid m$, and the gate $g^0_j(x)=\mathbb 1\{\|x\|\in I_j\}$ satisfies (ii) for any partition $\{I_j\}_{j=1}^K$ of $[0,\infty)$ into intervals with $c_j\in I_j$.
 
\begin{theorem}\label{thm:lower_bound}
Under Assumption~\ref{assump:data} with Gaussian noise and Assumption~\ref{assump:lower}, let $\mathcal H$ be any of $\hd$, $\hdk$ $(1\le k\le K)$ or $\hdp$ $(p\in(0,1))$. If $n\ge\sigma^2Kd_{\mathrm{in}}d/R^2$, then
\begin{align}\label{eqn:lower_bound}
\inf_{\hat f}\sup_{f\in\mathcal H}\mathbb E\big[\|\hat f-f\|_n^2\big]\ge c\,\lambda\,\frac{\sigma^2Kd_{\mathrm{in}}d_\out}{n},
\end{align}
where $c>0$ is an absolute constant and the infimum is over all estimators.
\end{theorem}

For linear experts, $d_\pre=d_{\mathrm{in}}d_\out$, so  Eq.\eqref{eqn:lower_bound} matches the leading term $Kd_\pre/n$ of Theorems~\ref{thm:complexity_general}, ~\ref{thm:complexity_top_k} \&~\ref{thm:complexity_top-p} up to logarithmic factors, whenever the gating and routing terms are of no larger order, i.e.\ $d_\prg+d_N k\log(eK/k)\lesssim Kd_\pre$ (for instance, $d_\prg=K(d_{\mathrm{in}}+1)$ for linear gating). Because the hard instances are well specified, Lemma~\ref{lem:oracle} turns the upper bounds on $\phi$ into risk bounds $\mathbb E\|\hat f-f\|_n^2\le\frac43\phi$ for these instances, which Eq.\eqref{eqn:lower_bound} shows are tight up to logarithmic factors. In particular, sparse routing does not reduce the worst-case sample complexity: every expert is active on its own group of inputs, and all $K$ of them must be estimated. The proof is in Appendix~\ref{proof:lower_bound}.

\section{EXPERIMENTS} \label{sec:experiments}
\begin{figure}
    \centering
    \includegraphics[width=0.9\linewidth]{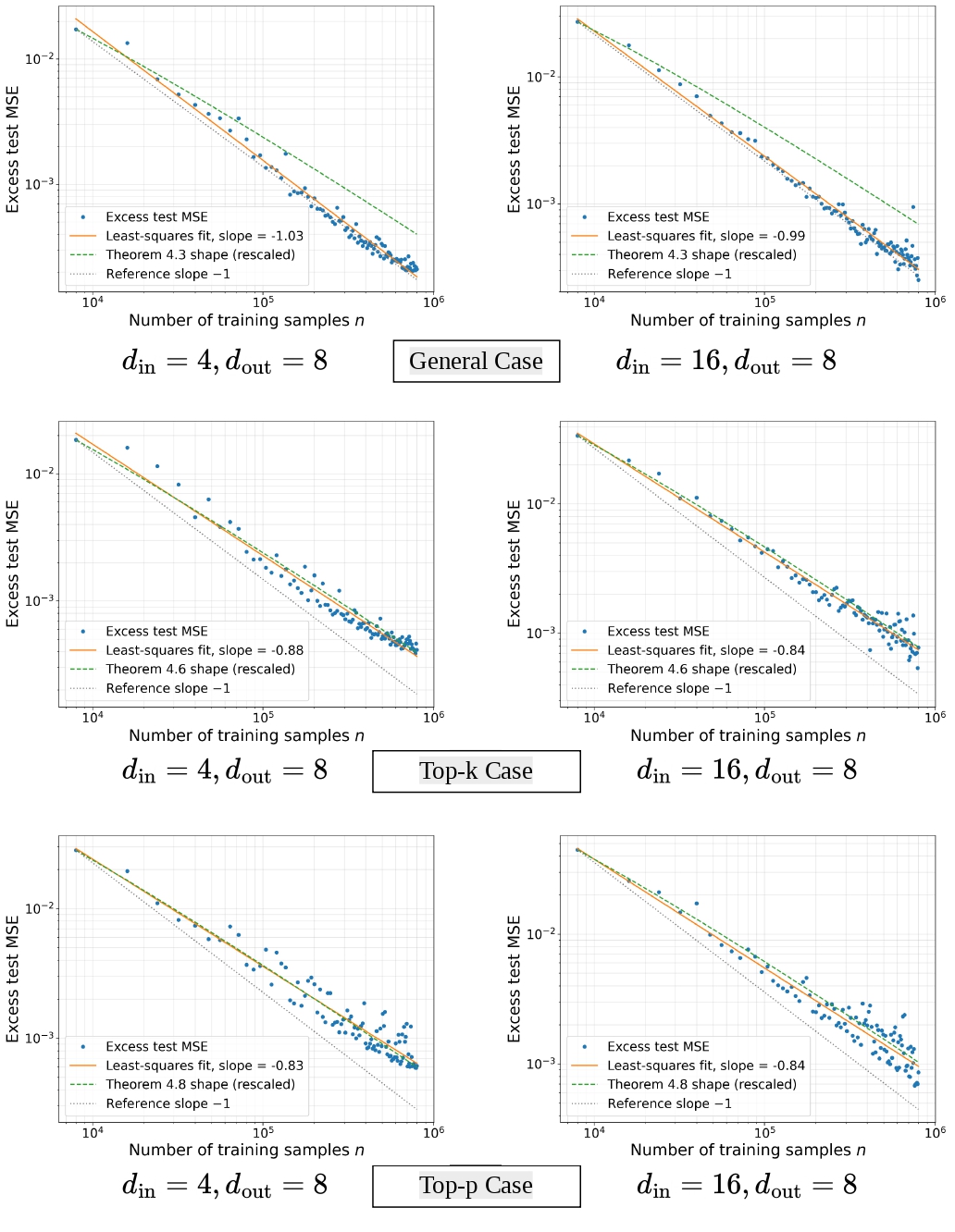}
    \caption{Excess test MSE versus $n$ (log--log) for the General, Top$-k$ and Top-$p$ cases for $2$ $(d_\iin, d_\out)$ combinations.}
    \label{fig:main_loglog}
\end{figure}

We empirically investigate the approximation and generalization properties of a feed-forward MoE model for learning non-linear mappings induced by a fixed Transformer Encoder, examining whether observed test RMSE behaviors are consistent with the theoretical bounds in Eq.~\eqref{eqn:complexity_general}, \eqref{eqn:complexity_top-k}, and \eqref{eqn:complexity_top-p}. We construct synthetic regression problems with covariates drawn from a standard multivariate Gaussian and targets generated by a fixed Transformer Encoder plus small Gaussian noise, varying input dimensions $d_\iin\in\{4,8,16,32\}$ and output dimensions $d_\out\in\{8,16,32\}$, yielding $12$ input-output configurations per setting. The detailed codebase is available at \href{https://anonymous.4open.science/r/Mixture_of_Experts_Sparsity_Compute-8024/README.md}{This Annonymous Repository}.\\
\noindent 
\textbf{Synthetic Data Generation}
For each configuration, we generate $N=10^6$ samples with $X_i\sim\mathcal{N}(0,I_{d_\pr})$ and regression targets $Y_i = \text{Transformer}(X_i)+\varepsilon_i$, where $\varepsilon_i\sim\mathcal{N}(0,\sigma^2 I_{d_\out})$ with $\sigma=5\times10^{-4}$. The Transformer encoder is frozen ($8$ heads, hidden dim $256$, $8$ layers), defining a fixed non-linear target function. Data are standardized componentwise and split 4:1 into train and test sets.\\
\noindent 
\textbf{MoE Architecture}
We use a feed-forward MoE with $K=4$ experts, each a $1$hidden-layer network (hidden dim $32$), and a $1$-hidden-layer gating network with softmax output. Except for ablations, the value of $k$ in top-$k$ has been kept fixed at $k=2$, and the value of $p$ in top-$p$ has been fixed at $p=0.7$. All networks use GELU activations. The model is trained with AdamW ($\eta=10^{-3}$, batch size $1024$) minimizing RMSE over $50$ epochs per training-size increment.\\
\noindent 
\textbf{Evaluation Metrics.}
To study how generalization error varies with training size, we partition the available data into $100$ equally spaced increments $n_j = \lfloor \frac{j}{100}\cdot\text{train size}\rfloor$ for $j=1,\ldots,100$, spanning $1\%$ to $100\%$ of the training data. At each increment, we record the test MSE, which is our primary measure of generalization. Since the theorems bound the estimation error rather than the total test error, we subtract an estimated irreducible floor (noise and approximation error) $\hat a$ obtained by fitting $\mathrm{MSE}(n)=a+b\,n^{-\beta}$ to the test MSE, and we study the resulting \emph{excess test MSE} $\mathrm{MSE}(n)-\hat a$ on a log--log scale. We report the least-squares slope of the excess test MSE and compare it with the slope of the corresponding bound shape, i.e., the right-hand side of the bound with $L=1$, rescaled to pass through the first data point (so that only its slope is meaningful), evaluated over the same range of $n$. We treat the general case \eqref{eqn:complexity_general}, the top-$k$ case \eqref{eqn:complexity_top-k}, and the top-$p$ case \eqref{eqn:complexity_top-p} in turn. The floor-fitting procedure, the exact bound shapes, and results for all configurations are given in Appendix~\ref{app:experiments}.\\
\noindent
\textbf{Empirical rate in $n$.}
Figure~\ref{fig:main_loglog} plots the excess test MSE against the number of training samples $n$ (log--log) for the general, Top-$k$ ($k=2$) and Top-$p$ ($p=0.7$) MoE with $K=4$, for two input--output configurations, $(d_{\mathrm{in}},d_{\mathrm{out}})=(4,8)$ and $(16,8)$; the remaining ten configurations are in Appendices ~\ref{app:general_experiments}, ~\ref{app:topk_experiments} \& ~\ref{app:topp_experiments}. The irreducible part (noise and approximation error) is removed by a fitted floor, and each dashed-curve is the shape of the corresponding bound (Theorems~\ref{thm:complexity_general}, \ref{thm:complexity_top_k} and \ref{thm:complexity_top-p}; for Top-$k$, without the routing term), rescaled to the first data point, so that only its slope is meaningful. In all three routing schemes, the excess error decays at a rate close to $n^{-1}$. For $(4,8)$ the fitted slopes are $-1.03$, $-0.88$ and $-0.83$ for the general, Top-$k$ and Top-$p$ models, against bound-shape slopes of $-0.84$, $-0.85$ and $-0.83$. For $(16,8)$ they are $-0.99$, $-0.84$ and $-0.84$, against $-0.82$, $-0.83$ and $-0.80$. Thus in the two configurations shown, the measured decay is at least as steep as that of the bound, as expected of an upper bound, although the margins for Top-$k$ and Top-$p$ are within fitting error. This does not hold in every configuration: over all twelve, the mean fitted slopes are $-1.00$ (general), $-0.83$ (Top-$k$) and $-0.80$ (Top-$p$), and several Top-$k$ runs decay more slowly than the bound shape (e.g. $-0.62$ for $(32,8)$), which we attribute to optimization error and to the difficulty of identifying the floor and exponent jointly (Appendix~\ref{app:experiments}). \\
\noindent
\textbf{Comparison with Top-$k$.}
The mean empirical slope is $-0.81$ for Top-$p$ and $-0.83$ for Top-$k$, and the configurations that deviate most are largely shared between the two routing rules. This agrees with the remark following Theorem~\ref{thm:complexity_top-p}: with at most $\kappa$ active experts, Top-$p$ has the same statistical rate as Top-$k$ with $k=\kappa$, and both scale with the total parameter count $D$ rather than with the number of active experts.\\
\noindent
\paragraph{Routing Behavior of Top-$p$.}
Table~\ref{tab:shannon_entropy_top-p} reports, for the final models ($n=8\times10^5$, $K=4$, $p=0.7$) evaluated on the test set, the average Shannon entropy of the gate distribution and the average number of active experts $\bar k_x$, where $k_x$ is the smallest number of experts whose cumulative gate probability reaches $p$. The entropies are small in all configurations (at most $0.166$, against a maximum of $\log 4\approx 1.386$), and $\bar k_x$ is close to $1$ (e.g., $1.03$ for $(d_{\mathrm{in}},d_{\mathrm{out}})=(4,8)$, and $1.07$ for $(d_\iin, d_\out) = (16,8)$). Since $k_x\ge 1$, the fraction of test inputs with $k_x\ge 2$ is at most $\bar k_x-1$, so the learned router is nearly hard. Figure~\ref{fig:topp_active_experts} shows that for $(4,8)$, $\bar k_x$ decreases from about $1.21$ at $n=8\times10^3$ to about $1.03$ at $n=8\times10^5$ and for $(16,8)$ the same drops from $1.37$ at $n=8\times 10^3$ to about $1.07$ at $n=8\times 10^5$, flattening for $n\gtrsim10^5$ [See Appendix~\ref{app:top-p_shannon_full} for all the $12$ combinations]. This is far below the worst-case count $\lceil Kp\rceil=3$ used in our bound, which is therefore conservative for trained routers. It also means that our experiments exercise multi-expert routing only rarely, so they measure the rate in $n$ and do not probe the dependence of the bound on $p$ or on the number of active experts. We report the decrease of $\bar k_x$ as an observation and do not attribute it to a specific cause (it may reflect gate sharpening \citep{chen2022towards} or optimization dynamics under a fixed number of epochs). Each point is a single training run.

\begin{table}[h]
\caption{Top-$p$ MoE ($K=4$, $p=0.7$), final models ($n=8\times10^5$), test set: Shannon entropy of the gate distribution / average number of active experts $\bar k_x$.}
\label{tab:shannon_entropy_top-p}
\begin{center}
\small
\begin{tabular}{l|lll}
$d_{\mathrm{in}}$ & \multicolumn{3}{c}{$d_{\mathrm{out}}$} \\
\cline{2-4}
 & $\qquad 8$ & $\qquad 16$ & $\qquad 32$ \\
\hline
$4$  & $0.020$ / $1.03$ & $0.032$ / $1.05$ & $0.073$ / $1.11$ \\
$8$  & $0.037$ / $1.06$ & $0.052$ / $1.08$ & $0.032$ / $1.05$ \\
$16$ & $0.043$ / $1.07$ & $0.070$ / $1.11$ & $0.108$ / $1.17$ \\
$32$ & $0.026$ / $1.04$ & $0.062$ / $1.09$ & $0.166$ / $1.12$ \\
\end{tabular}
\end{center}
\end{table}

\begin{figure}[h]
    \centering
    \includegraphics[width=1.0\linewidth]{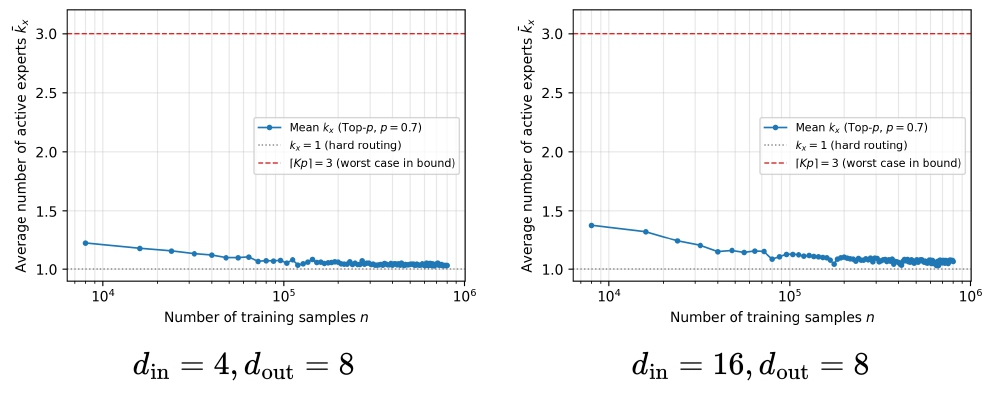}
    \caption{Average number of experts selected by Top-$p$ routing ($K=4$, $p=0.7$) on the test set versus training size, for $(d_{\mathrm{in}},d_{\mathrm{out}})=(4,8) \& (16,8)$. Dotted: hard routing ($k_x=1$); dashed: worst-case count $\lceil Kp\rceil=3$ used in the bound.}
    \label{fig:topp_active_experts}
\end{figure}

\section{CONCLUSION}

In this work, we studied the estimation error of least-squares Mixture-of-Experts (MoE) estimators whose experts and gating network are parametric Lipschitz. We represented each MoE as a gate-expert pair in the \emph{Gated Hilbert Space}, whose norm $\|\cdot\|_\Delta$ separates both routing and expert discrepancies. We used its empirical representation to derive an oracle inequality with a fast rate of order $\tilde{\mathcal O}\big((Kd_{\mathrm{pr},e}+d_{\mathrm{pr},g})/n\big)$ for dense MoEs with $K$ experts.

Our results show that sparsity saves compute but not samples: the Top-$k$ and Top-$p$ bounds do not improve on the dense rate $\tilde{\mathcal O}\big((Kd_{\mathrm{pr},e}+d_{\mathrm{pr},g})/n\big)$ and instead carry an additive routing cost governed by the Natarajan dimension of the selector. This cost arises because the selection map is discontinuous in the gating parameters, so Lipschitz continuity alone does not control it; we handle it by partitioning the gating class according to the selection patterns it induces on the sample, and we show that a routing-complexity condition cannot be dropped. For linear softmax gating, the routing cost is of the same order as the gating term. For linear experts, a matching minimax lower bound of order $Kd_{\mathrm{pr},e}/n$ shows that the rates are optimal up to logarithmic factors.

Experiments on synthetic regression tasks show that the excess test MSE decays at a rate close to $n^{-1}$ for general, Top-$k$ and Top-$p$ routing, in line with our bounds. They check the rate in $n$ only: they do not test the dependence on $K$, $k$ or $p$, and the learned Top-$p$ routers are nearly hard ($\bar k_x\approx 1$), so multi-expert routing is rarely exercised.

Several questions remain open. The Natarajan dimension of practical gating networks is not computable and is characterized here only for linear softmax gating, and the approximation error under misspecification is not controlled. Experiments varying $K$, $k$ and $p$, with routers that use several experts adaptively, would test the dependence of the bounds on these quantities. We hope these results provide a foundation for sharper bounds on adaptive routing in MoE models.

\subsection*{AI use statement}

In this work, we used generative AI tools for language editing, structural reorganization, and algebraic cross-checking. We have not used generative AI tools to create research content. Additionally, we used generative AI tools to check the coding experiments described in the paper. We have reviewed all AI-assisted work. We take responsibility for the final content of this work, including text, claims, or artifacts produced with the aid of generative AI.

\bibliographystyle{plainnat}
\bibliography{aistats_ref_main}

\section*{Checklist}

The checklist follows the references. For each question, choose your answer from the three possible options: Yes, No, Not Applicable.  You are encouraged to include a justification to your answer, either by referencing the appropriate section of your paper or providing a brief inline description (1-2 sentences). 
Please do not modify the questions.  Note that the Checklist section does not count towards the page limit. Not including the checklist in the first submission won't result in desk rejection, although in such case we will ask you to upload it during the author response period and include it in camera ready (if accepted).

\textbf{In your paper, please delete this instructions block and only keep the Checklist section heading above along with the questions/answers below.}

\begin{enumerate}

  \item For all models and algorithms presented, check if you include:
  \begin{enumerate}
    \item A clear description of the mathematical setting, assumptions, algorithm, and/or model. [Yes] 
    [\textbf{Description:} Section 3 discusses the mathematical preliminaries of various routing mechanisms and the notations used]
    \item An analysis of the properties and complexity (time, space, sample size) of any algorithm. [Yes] 
    [\textbf{Description:} Section 4 talks about all the propositions, lemmas, and theorems claimed to be true in the work; while the proofs are deferred to the Supplementary Document with appropriate pointer.]
    \item (Optional) Anonymized source code, with specification of all dependencies, including external libraries. [Yes] 
    [\textbf{Description:} The link to the Annonymized GitHub Repository has been provided in Section 5]
  \end{enumerate}

  \item For any theoretical claim, check if you include:
  \begin{enumerate}
    \item Statements of the full set of assumptions of all theoretical results. [Yes][\textbf{Description:} All such statements have been discussed in detail in Section 4.]
    \item Complete proofs of all theoretical results. [Yes] [\textbf{Description:} The proofs have been detailed out in the Supplementary.]
    \item Clear explanations of any assumptions. [Yes] [\textbf{Description:} The assumptions have been made explicit in Section 3.]     
  \end{enumerate}

  \item For all figures and tables that present empirical results, check if you include:
  \begin{enumerate}
    \item The code, data, and instructions needed to reproduce the main experimental results (either in the supplemental material or as a URL). [Yes] [\textbf{Description:} The Annonymized GitHub Respository has been linked to Section 5]
    \item All the training details (e.g., data splits, hyperparameters, how they were chosen). [Yes] [\textbf{Description:} Sections 2,5,6 \& 7 in the Repository discuss about them in detail.]
    \item A clear definition of the specific measure or statistics and error bars (e.g., with respect to the random seed after running experiments multiple times). [Yes] [\textbf{Description:} Sections 11 \& 14 in the Repository discuss about these points.]
    \item A description of the computing infrastructure used. (e.g., type of GPUs, internal cluster, or cloud provider). [Yes] [\textbf{Description:} Sections 12 \& 14 lay out the details of the hardware and environments used.]
  \end{enumerate}

  \item If you are using existing assets (e.g., code, data, models) or curating/releasing new assets, check if you include:
  \begin{enumerate}
    \item Citations of the creator If your work uses existing assets. [Not Applicable]
    \item The license information of the assets, if applicable. [Not Applicable]
    \item New assets either in the supplemental material or as a URL, if applicable. [Not Applicable]
    \item Information about consent from data providers/curators. [Not Applicable]
    \item Discussion of sensible content if applicable, e.g., personally identifiable information or offensive content. [Not Applicable]
  \end{enumerate}

  \item If you used crowdsourcing or conducted research with human subjects, check if you include:
  \begin{enumerate}
    \item The full text of instructions given to participants and screenshots. [Not Applicable]
    \item Descriptions of potential participant risks, with links to Institutional Review Board (IRB) approvals if applicable. [Not Applicable]
    \item The estimated hourly wage paid to participants and the total amount spent on participant compensation. [Not Applicable]
  \end{enumerate}

\end{enumerate}

\clearpage
\appendix
\thispagestyle{empty}


\onecolumn
\aistatstitle{Supplementary Materials for: \\
 Sparsity Saves Compute, Not Samples: Oracle Inequalities for Dense and Sparse Mixture of Experts}
\tableofcontents
\section{Supplementary to \emph{Background}}\label{supp:Background}

\noindent 
\begin{definition}\label{defn:parametric_lipschitz}

A parametrized class of real-valued functions 
\begin{align*}
    F = \left\{ f(\theta;\cdot) \colon \mathcal{X} \subset \rr^{d_\iin}\rightarrow \mathbb{R}^{d_\out}, \theta \in \Theta \subset \rr^{d_\pr}\ \right\}, 
\end{align*}

with $\Theta$ and $\mathcal{X}$ pre-compact, is said to be $L$-Lipschitz if, for all $f \in F$ and $x\in\xx$,
\begin{align*}
    \|f(\theta;x) - f(\theta^\prime;x)\|\leq L_x \|\theta - \theta^\prime\|_2.
\end{align*}

Since $\xx$ is pre-compact, we can uniformly replace $L_x$ with $L$ over the closure of $\xx$ (which is compact). Therefore, we have for all $f\in F$ and for all $x\in\xx$,
\begin{align*}
     \|f(\theta;x) - f(\theta^\prime;x)\|\leq L \|\theta - \theta^\prime\|_2.
\end{align*}

\end{definition}

\noindent
\begin{prop}[Completeness of $\hd$]
\label{prop:completeness}
Let $(\mathcal X\subseteq \rr^{d_\pr}, \mu)$ be a measure space and let $\hd$ be the space of gated MoE predictors defined in Eq.~\ref{eqn:defn_general_moe}
equipped with the norm induced by the inner product defined in Eq.~\eqref{eqn:gated_inner_product}, and after identifying elements $\tilde f,\tilde r$ such that $g_i(x)a_i(x) = h_i(x)b_i(x)$ for $\mu$-almost every $x$ and every $i=1,\ldots,K$.
Then $(\hd, \langle \cdot,\cdot\rangle_\Delta)$ is complete with respect to the metric $d_\Delta(\tilde f,\tilde r) = \|\tilde f - \tilde r\|_\Delta$. Consequently, $\hd$ is a (gated/modulated)  Hilbert space.

\end{prop}

\begin{proof}
    Let $H \coloneqq L^2(\mathcal X,\mu;\mathbb R^{d_{\mathrm{out}}})^K$ denote the Hilbert space of $K$-tuples of square-integrable $\mathbb R^{d_{\mathrm{out}}}$-valued functions, equipped with the inner product
\[
\langle \varphi,\psi \rangle_H \coloneqq \sum_{i=1}^K \int_{\mathcal X} \langle \varphi_i(x),\psi_i(x)\rangle \, d\mu(x), \qquad \varphi=(\varphi_1,\ldots,\varphi_K),\ \psi = (\psi_1,\ldots,\psi_K) \in H.
\]

\medskip
\noindent \textbf{Step 1: Construction of a map $\Phi$.}
Define
\begin{align*}
    \Phi : \hd / \!\sim \; \longrightarrow \; H, \qquad \Phi(\tilde f) \coloneqq (g_1 a_1, \ldots, g_K a_K),
\end{align*}

where $\sim$ is the equivalence relation identifying $\tilde f \sim \tilde r$ whenever $g_ia_i = h_ib_i$ $\mu$-a.e.\ for all $i$. This map is well defined on the quotient by construction.

\medskip
\noindent \textbf{Step 2: $\Phi$ is a linear isometry.}
By definition of the bilinear form in \eqref{eqn:gated_inner_product}
\begin{align*}
    \langle \tilde f, \tilde r \rangle_\Delta = \int_{\mathcal X} \sum_{i=1}^K \langle g_i(x)a_i(x), h_i(x)b_i(x)\rangle \, d\mu(x) = \langle \Phi(\tilde f), \Phi(\tilde r) \rangle_H,
\end{align*}

so $\Phi$ preserves the inner product exactly, and in particular
\begin{align*}
    \|\tilde f\|_\Delta = \|\Phi(\tilde f)\|_H.
\end{align*}

Linearity of $\Phi$ is immediate from linearity of pointwise products in $(g_i,a_i)$ upon fixing representatives, and injectivity on $\hd/\!\sim$ follows directly from the definition of $\sim$: if $\Phi(\tilde f) = \Phi(\tilde r)$ then $g_ia_i = h_ib_i$ $\mu$-a.e.\ for all $i$, i.e.\ $\tilde f \sim \tilde r$.

\medskip
\noindent \textbf{Step 3: $\Phi$ is surjective.}
Let $\varphi = (\varphi_1,\ldots,\varphi_K) \in L$ be arbitrary. Define
\begin{align*}
 g_i(x) \equiv \frac{1}{K}, \qquad a_i(x) \coloneqq K\varphi_i(x), \qquad i=1,\ldots,K.
\end{align*}

Then $g_i:\mathbb R^{d_{\mathrm{in}}}\to[0,1]$ with $\sum_{i=1}^K g_i(x) = 1$ for all $x$, and each $a_i:\mathbb R^{d_{\mathrm{in}}}\to\mathbb R^{d_{\mathrm{out}}}$ is a valid expert map. Hence $\tilde f = (\tilde g,\tilde a) \in \hd$, and
\begin{align*}
 \Phi(\tilde f) = (g_1a_1,\ldots,g_Ka_K) = (\varphi_1,\ldots,\varphi_K) = \varphi.
\end{align*}

Thus $\Phi$ is onto.

\medskip
\noindent \textbf{Step 4: $\Phi$ is a bijective linear isometry.}
Combining Steps 2 and 3, $\Phi: \hd/\!\sim \; \to H$ is a linear bijection satisfying $\langle \Phi(\tilde f),\Phi(\tilde r)\rangle_H = \langle \tilde f,\tilde r\rangle_\Delta$ for all $\tilde f, \tilde r$. Hence $\Phi$ is an isometric isomorphism of inner-product spaces.

\medskip
\noindent \textbf{Step 5: Transfer of Completeness through the isometry.}
By the Riesz--Fischer theorem, $L^2(\mathcal X,\mu;\mathbb R^{d_{\mathrm{out}}})$ is complete, and a finite product of complete inner-product spaces (with the product inner product) is complete; hence $H$ is a Hilbert space.

Let $(\tilde f_n)_{n\ge 1}$ be a $d_\Delta$-Cauchy sequence in $\hd$. By Step 2,
\begin{align*}
    \|\Phi(\tilde f_n) - \Phi(\tilde f_m)\|_H = \|\tilde f_n - \tilde f_m\|_\Delta \longrightarrow 0 \quad \text{as } n,m\to\infty,
\end{align*}

so $(\Phi(\tilde f_n))_{n\ge 1}$ is Cauchy in $H$. By completeness of $H$, there exists $\varphi = (\varphi_1,\ldots,\varphi_K) \in H$ such that $\Phi(\tilde f_n) \to \varphi$ in $H$.

By surjectivity of $\Phi$ (Step 3), there exists $\tilde f \in \hd$, e.g.
\begin{align*}
  \tilde f = \left(\tfrac{1}{K},\ldots,\tfrac{1}{K},\, K\varphi_1,\ldots,K\varphi_K\right),
\end{align*}

such that $\Phi(\tilde f) = \varphi$. Then, using the isometry property once more,
\begin{align*}
   \|\tilde f_n - \tilde f\|_\Delta = \|\Phi(\tilde f_n) - \Phi(\tilde f)\|_H = \|\Phi(\tilde f_n) - \varphi\|_H \longrightarrow 0.
\end{align*}

Hence $\tilde f_n \to \tilde f$ in $d_\Delta$, with $\tilde f \in \hd$.

\medskip

Every $d_\Delta$-Cauchy sequence in $\hd$ converges to an element of $\hd$. Therefore $(\hd, \langle\cdot,\cdot\rangle_\Delta)$ is complete, and being an inner-product space that is complete with respect to the induced metric, $\hd/\sim$ is a Hilbert space; in other words, $\hd$ is a (gated/modulated) Hilbert Space. 
\end{proof}

\section{Supplementary to \emph{Estimation Error Rates For Various Parametric MoE Classes}}

\subsubsection{Proofs of Lemma~\ref{lem:oracle} }
\begin{proof}[Proof of Lemma~\ref{lem:oracle}]
Let $f^\ast\in\hd$ attain $\inf_{f\in\hd}\|f-f_\rho\|_n^2$ (otherwise take a near-minimizer and pass to the limit), and let $h:=\hat f-f^\ast\in\hd-\hd$. Since $\hat f$ minimizes the empirical squared loss, expanding $Y_t=f_\rho(x_t)+\xi_t$ gives the basic inequality
\[
\|\hat f-f_\rho\|_n^2\le\|f^\ast-f_\rho\|_n^2+\frac2n\sum_{t=1}^n\langle\xi_t,h(x_t)\rangle .
\]
Write $Z$ for the supremum in \eqref{eqn:offset_def}, so that $\frac2n\sum_t\langle\xi_t,h(x_t)\rangle\le Z+\frac14\|h\|_n^2$. Using $\|h\|_n^2\le2\|\hat f-f_\rho\|_n^2+2\|f^\ast-f_\rho\|_n^2$,
\[
\|\hat f-f_\rho\|_n^2\le\tfrac32\|f^\ast-f_\rho\|_n^2+\tfrac12\|\hat f-f_\rho\|_n^2+Z .
\]
Rearranging and taking expectations (the design is fixed, so $\|f^\ast-f_\rho\|_n^2$ is deterministic) gives \eqref{eq:oracle_ineq}. If $f_\rho\in\hd$, take $f^\ast=f_\rho$; then $\|h\|_n^2=\|\hat f-f_\rho\|_n^2\le Z+\frac14\|h\|_n^2$, so $\|\hat f-f_\rho\|_n^2\le\frac43 Z$.
\end{proof}

\subsection{General MoE Class}
\begin{prop}\label{prop:general_entropy_estimate}
Suppose the gating and expert classes are parametric $L$-Lipschitz and the experts are uniformly bounded, $B_\infty:=\sup_{\tilde a\in\mathcal E}\max_{i\le K}\sup_{x\in\xx}\|a_i(x)\|<\infty$. Then, for any probability measure $\mu$ on $\xx$,
\begin{align*}
\mathcal N\left(\epsilon,\hd,d_\Delta\right)\lesssim
\left(\frac{2\sqrt2\,LG_\ast \sqrt{d_\out}}{\epsilon}\right)^{d_\pre K}
\left(\frac{2\sqrt2\,L \sqrt{K}\,B_\infty}{\epsilon}\right)^{d_\prg}.
\end{align*}
\end{prop}

\begin{proof}\label{proof:general_entropy_estimate}
\textbf{Step 1: Expert class.}
Write $\mathcal E=\{\tilde a=(a_1,\ldots,a_K)\}$ with $a_i=(a^i_1,\ldots,a^i_{d_\out}):\xx\to\mathbb R^{d_\out}$, and for $\tilde a,\tilde b\in\mathcal E$ let
\begin{align*}
\|\tilde a-\tilde b\|_{\mathcal E,\infty}:=\max_{1\le i\le K}\max_{1\le j\le d_\out}\|a^i_j-b^i_j\|_\infty,
\qquad\text{so that}\qquad
\sup_{x}\|a_i(x)-b_i(x)\|\le\sqrt{d_\out}\,\|\tilde a-\tilde b\|_{\mathcal E,\infty}.
\end{align*}
Each expert is an $L$-Lipschitz function of its $d_\pre$ parameters, and the $K$ experts have separate parameters. The standard covering estimate for parametric Lipschitz families \citep{bartlett2013lecture14}, in which $\lesssim$ hides the diameter of the parameter set, gives
\begin{align}\label{eq:param_expert_cover}
\mathcal N\left(\delta,\mathcal E,\|\cdot\|_{\mathcal E,\infty}\right)\le\prod_{i=1}^K\mathcal N\left(\delta,\mathcal E_i,\|\cdot\|_\infty\right)\lesssim\left(\frac{L}{\delta}\right)^{d_\pre K},
\end{align}
where $\mathcal E_i$ denotes the class of the $i$-th expert.

\textbf{Step 2: Gating class.}
For $\tilde g,\tilde h\in\mathcal G$ let $\|\tilde g-\tilde h\|_{\mathcal G,\infty}:=\max_{1\le i\le K}\|g_i-h_i\|_\infty$. In the same way,
\begin{align}\label{eq:param_gate_cover}
\mathcal N\left(\delta,\mathcal G,\|\cdot\|_{\mathcal G,\infty}\right)\lesssim\left(\frac{L}{\delta}\right)^{d_\prg}.
\end{align}

\textbf{Step 3: MoE class.}
Let $\tilde{f}=(\tilde g,\tilde a)$ and $\tilde{r}=(\tilde h,\tilde b)\in\hd$. Inserting $w=g_ib_i$ in \eqref{eqn:distance} and using $\|u-v\|^2\le2\|u-w\|^2+2\|w-v\|^2$ together with $\sqrt{s+t}\le\sqrt s+\sqrt t$,
\begin{align}\label{eqn:moe_split}
d_\Delta(\tilde{f},\tilde{r})
\le\sqrt2\left(\int_\xx\sum_{i=1}^K g_i(x)^2\|a_i(x)-b_i(x)\|^2d\mu(x)\right)^{1/2}
+\sqrt2\left(\int_\xx\sum_{i=1}^K\|b_i(x)\|^2|g_i(x)-h_i(x)|^2d\mu(x)\right)^{1/2}.
\end{align}
The first term separates the expert discrepancy, weighted by the gates, and the second the routing discrepancy, weighted by the expert magnitudes. Since $\sum_i\|b_i(x)\|^2\le KB_\infty^2$ for every $x$, and $\mu$ is a probability measure,
\[
d_\Delta(f,r)\le\sqrt2\,G_\ast\sqrt{d_\out} \|\tilde a-\tilde b\|_{\mathcal E,\infty}+\sqrt2\,\sqrt {K} \,B_\infty\,\|\tilde g-\tilde h\|_{\mathcal G,\infty}, \quad G_\ast = \sup_{\tilde{g}\in\mathcal G}\left(\int_\xx\sum_{i=1}^Kg_i(x)^2d\mu(x)\right)\leq 1
\]
Hence $d_\Delta(f,r)\le\epsilon$ whenever $\|\tilde a-\tilde b\|_{\mathcal E,\infty}\le\delta_E:=\epsilon/(2\sqrt2G_\ast \sqrt{d_\out})$ and $\|\tilde g-\tilde h\|_{\mathcal G,\infty}\le\delta_G:=\epsilon/(2\sqrt2\sqrt K B_\infty)$, so the product of a $\delta_E$-cover of $\mathcal E$ and a $\delta_G$-cover of $\mathcal G$ is an $\epsilon$-cover of $\hd$. By \eqref{eq:param_expert_cover} and \eqref{eq:param_gate_cover},
\begin{align*}
\log\mathcal N\left(\epsilon,\hd,d_\Delta\right)
&\le\log\mathcal N\left(\delta_E,\mathcal E,\|\cdot\|_{\mathcal E,\infty}\right)+\log\mathcal N\left(\delta_G,\mathcal G,\|\cdot\|_{\mathcal G,\infty}\right)\\
&\lesssim d_\pre K\log\frac{2\sqrt2\,LG_\ast \sqrt{d_\out}}{\epsilon}+d_\prg\log\frac{2\sqrt2\,L \sqrt{K}\,B_\infty}{\epsilon},
\end{align*}
finishing the proof of the proposition. 
\end{proof}

\noindent
\subsubsection{Proof of Theorem \ref{thm:complexity_general}}
\begin{proof}\label{proof:complexity_general}
An element $\tilde{f}=(\tilde g,\tilde a)\in\hd$ computes $\tilde{f}(x)=\sum_{i=1}^K g_i(x)a_i(x)$. We use the norms $\|\cdot\|_{\mathcal E,\infty}$ and $\|\cdot\|_{\mathcal G,\infty}$ of Proposition~\ref{prop:general_entropy_estimate}, for which $\sup_x\|a_i(x)-b_i(x)\|\le\sqrt{d_\out}\,\|\tilde a-\tilde b\|_{\mathcal E,\infty}$, together with the parametric covering estimates \eqref{eqn:topk_base}.

\textbf{Step 1: Pointwise decomposition.}
Fix $x_1,\ldots,x_n\in\xx$ and let $\rho_n(\tilde{f},\tilde{r}):=\max_{t\le n}\|\tilde{f}(x_t)-\tilde{r}(x_t)\|$, so that $\|\tilde{f}-\tilde{r}\|_{L^2(P_n)}\le\rho_n(\tilde{f},\tilde{r})$. For $\tilde{f}=(\tilde g,\tilde a)$, $\tilde{r}=(\tilde h,\tilde b)\in\hd$ and any $t$, inserting $g_ib_i$ as in \eqref{eqn:moe_split} separates the expert and routing discrepancies:
\begin{align*}
\|\tilde{f}(x_t)-\tilde{r}(x_t)\|
&\le\sum_{i=1}^K g_i(x_t)\|a_i(x_t)-b_i(x_t)\|+\sum_{i=1}^K|g_i(x_t)-h_i(x_t)|\,\|b_i(x_t)\|\\
&\le\sqrt{d_\out}\,\|\tilde a-\tilde b\|_{\mathcal E,\infty}+KB_\infty\,\|\tilde g-\tilde h\|_{\mathcal G,\infty},
\end{align*}
using $\sum_i g_i(x_t)=1$. Hence
\begin{align}\label{eqn:dense_lip}
\rho_n(\tilde{f},\tilde{r})\le\sqrt{d_\out}\,\|\tilde a-\tilde b\|_{\mathcal E,\infty}+KB_\infty\,\|\tilde g-\tilde h\|_{\mathcal G,\infty}.
\end{align}

\textbf{Step 2: Entropy.}
By \eqref{eqn:dense_lip}, $\rho_n(f,r)\le\epsilon$ whenever $\|\tilde a-\tilde b\|_{\mathcal E,\infty}\le\epsilon/(2\sqrt{d_\out})$ and $\|\tilde g-\tilde h\|_{\mathcal G,\infty}\le\epsilon/(2KB_\infty)$. The routing map of a dense MoE is continuous in the gating parameters, so no partition of the gating class is needed, and by \eqref{eqn:topk_base}, for every sample,
\begin{align}\label{eqn:dense_entropy}
\log\mathcal N\big(\epsilon,\hd,L^2(P_n)\big)\lesssim d_\pre K\log\frac{2L\sqrt{d_\out}}{\epsilon}+d_\prg\log\frac{2LKB_\infty}{\epsilon}.
\end{align}

\textbf{Step 3: Oracle inequality.}
Taking $\alpha=\gamma$ in \eqref{eqn:offset_rad} removes the entropy integral, so
\[
\phi(\hd,n)\le\inf_{\gamma>0}\Big\{\frac{\log\mathcal N\big(\gamma,\hd,L^2(P_n)\big)}{cn}+4\gamma\Big\}.
\]
By \eqref{eqn:dense_entropy}, the right-hand side equals $\frac{D\log(1/\gamma)}{cn}+4\gamma$ up to terms not depending on $\gamma$. Setting its derivative $-D/(cn\gamma)+4$ to zero gives $\gamma^\ast=D/(4cn)$. Substituting $\gamma^\ast$,
\[
\phi(\hd,n)\lesssim\frac{D}{n}+\frac{d_\pre K\log\!\big(\frac{8cL\sqrt{d_\out}\,n}{D}\big)+d_\prg\log\!\big(\frac{8cLKB_\infty\,n}{D}\big)}{n},
\]
and absorbing the absolute constant $8c$ into $\lesssim$ gives \eqref{eqn:complexity_general}.
\end{proof}

\noindent
\subsection{Top-k MoE Class}

\subsubsection{Proof of Theorem \ref{thm:complexity_top_k}}\label{proof:complexity_top_k}
\begin{proof}
\textbf{Notation.} An element $\tilde{f}=(\tilde g,\tilde a)\in\hdk$ computes
\[
    \tilde{f}(x)=\sum_{i\in\sigma_{\tilde g}(x)}\hat g_i(x)\,a_i(x),
    \qquad
    \hat g_i(x)=\frac{g_i(x)}{\sum_{j\in\sigma_{\tilde g}(x)}g_j(x)}.
\]
For $\tilde a,\tilde b\in\mathcal E$ and $\tilde g,\tilde h\in\mathcal G$, let
$\|\tilde a-\tilde b\|_{\mathcal E,\infty}:=\max_{i\le K}\max_{j\le d_\out}\|a_{ij}-b_{ij}\|_\infty$ and
$\|\tilde g-\tilde h\|_{\mathcal G,\infty}:=\max_{i\le K}\|g_i-h_i\|_\infty$. Then
$\sup_x\|a_i(x)-b_i(x)\|\le\sqrt{d_\out}\,\|\tilde a-\tilde b\|_{\mathcal E,\infty}$ for every $i$. By the standard covering estimate for parametric Lipschitz families \citep{bartlett2013lecture14},
\begin{align}\label{eqn:topk_base}
    \mathcal N(\delta,\mathcal E,\|\cdot\|_{\mathcal E,\infty})\lesssim\Big(\frac{L}{\delta}\Big)^{d_\pre K},
    \qquad
    \mathcal N(\delta,\mathcal G,\|\cdot\|_{\mathcal G,\infty})\lesssim\Big(\frac{L}{\delta}\Big)^{d_\prg}.
\end{align}
 
Since the map $\tilde g\mapsto\sigma_{\tilde g}(x)$ is discontinuous: two gates that are arbitrarily close in $\|\cdot\|_{\mathcal G,\infty}$ can select different experts near ties, and there the outputs differ by an amount that does not shrink with $\|\tilde g-\tilde h\|_{\mathcal G,\infty}$. We therefore fix the sample, split the gating class according to the selection pattern it induces on the sample, and cover each piece, on which selection is constant.
 
Fix $x_1,\ldots,x_n\in\xx$ and define $\rho_n(\tilde{f},\tilde{r}):=\max_{1\le t\le n}\|\tilde{f}(x_t)-\tilde{r}(x_t)\|$. Since $\|f-r\|_{L^2(P_n)}\le\rho_n(f,r)$,
\begin{align}\label{eqn:topk_rho}
    \mathcal N\big(\epsilon,\hdk,L^2(P_n)\big)\le\mathcal N\big(\epsilon,\hdk,\rho_n\big).
\end{align}
 
\textbf{Step 1: Partition by selection patterns.}
For $\tilde g\in\mathcal G$, let $\pi(\tilde g):=\big(\sigma_{\tilde g}(x_1),\ldots,\sigma_{\tilde g}(x_n)\big)$, and let $\Pi:=\{\pi(\tilde g):\tilde g\in\mathcal G\}$. Since $\mathcal S_k$ takes $\binom{K}{k}$ values and has Natarajan dimension $d_N$, Natarajan's lemma \citep{natarajan1989learning}; \citep[Lemma~29.4]{shalev2014understanding} gives
\begin{align}\label{eqn:topk_patterns}
    |\Pi|\le n^{d_N}\binom{K}{k}^{2d_N},
    \qquad\text{hence}\qquad
    \log|\Pi|\le d_N\Big(\log n+2k\log\frac{eK}{k}\Big),
\end{align}
using $\log\binom{K}{k}\le k\log(eK/k)$. For $\pi\in\Pi$, let $\mathcal G_\pi:=\{\tilde g\in\mathcal G:\pi(\tilde g)=\pi\}$ and $\mathcal H_\pi:=\{(\tilde g,\tilde a):\tilde g\in\mathcal G_\pi,\ \tilde a\in\mathcal E\}$. Then $\hdk=\bigcup_{\pi\in\Pi}\mathcal H_\pi$, so
\begin{align}\label{eqn:topk_union}
    \mathcal N(\epsilon,\hdk,\rho_n)\le|\Pi|\,\max_{\pi\in\Pi}\mathcal N(\epsilon,\mathcal H_\pi,\rho_n).
\end{align}
 
\textbf{Step 2: Perturbation bound inside one pattern.}
Fix $\pi\in\Pi$, $\tilde{f}=(\tilde g,\tilde a)$, $\tilde{r}=(\tilde h,\tilde b)\in\mathcal H_\pi$, and $t\le n$. Both models select the same set $A:=\sigma_{\tilde g}(x_t)=\sigma_{\tilde h}(x_t)$ at $x_t$, so no expert outside of the index set $A$ is selected by any one of them. Write $\delta:=\|\tilde g-\tilde h\|_{\mathcal G,\infty}$, and let $S:=\sum_{j\in A}g_j(x_t)$ and $S':=\sum_{j\in A}h_j(x_t)$. Using $\sum_{i\in A}\hat g_i(x_t)=1$ and Assumption~\ref{assump:topk}(i),
\begin{align}
    \|\tilde{f}(x_t)-\tilde{r}(x_t)\|
    &\le\sum_{i\in A}\hat g_i(x_t)\|a_i(x_t)-b_i(x_t)\|+\sum_{i\in A}|\hat g_i(x_t)-\hat h_i(x_t)|\,\|b_i(x_t)\| \nonumber\\
    &\le\sqrt{d_\out}\,\|\tilde a-\tilde b\|_{\mathcal E,\infty}+B_\infty\sum_{i\in A}|\hat g_i(x_t)-\hat h_i(x_t)|. \label{eqn:topk_split}
\end{align}
At $x_t$, we have $\hat g_i-\hat h_i=(g_i-h_i)/S+\hat h_i\,(S'-S)/S$. Since $\sum_{i\in A}\hat h_i=1$ and $|S-S'|\le|A|\,\delta=k\delta$,
\begin{align}\label{eqn:topk_weights}
    \sum_{i\in A}|\hat g_i-\hat h_i|\le\frac{1}{S}\Big(\sum_{i\in A}|g_i-h_i|+|S'-S|\Big)\le\frac{2k\delta}{S}\le 2K\delta,
\end{align}
where the last step uses $S\ge k/K$: the $k$ largest of $K$ nonnegative weights summing to one have total at least $k/K$. Substituting \eqref{eqn:topk_weights} into \eqref{eqn:topk_split} and maximizing over $t$,
\begin{align}\label{eqn:topk_lip}
    \rho_n(f,r)\le\sqrt{d_\out}\,\|\tilde a-\tilde b\|_{\mathcal E,\infty}+2KB_\infty\,\|\tilde g-\tilde h\|_{\mathcal G,\infty}
    \qquad\text{for all } f,r\in\mathcal H_\pi.
\end{align}
 
\textbf{Step 3: Covering one pattern class.}
By \eqref{eqn:topk_lip}, $\rho_n(f,r)\le\epsilon$ whenever $\|\tilde a-\tilde b\|_{\mathcal E,\infty}\le\epsilon/(2\sqrt{d_\out})$ and $\|\tilde g-\tilde h\|_{\mathcal G,\infty}\le\epsilon/(4KB_\infty)$, provided $\tilde g,\tilde h\in\mathcal G_\pi$. All $K$ experts must be covered, since different sample points may activate different $k$-subsets. For the gates, the cover centers must lie in $\mathcal G_\pi$. An internal $\delta$-cover of a subset costs at most an external $\delta/2$-cover of the whole set, so by \eqref{eqn:topk_base}
\begin{align}
    \mathcal N(\epsilon,\mathcal H_\pi,\rho_n)
    &\le\mathcal N\Big(\tfrac{\epsilon}{2\sqrt{d_\out}},\mathcal E,\|\cdot\|_{\mathcal E,\infty}\Big)\,
       \mathcal N\Big(\tfrac{\epsilon}{8KB_\infty},\mathcal G,\|\cdot\|_{\mathcal G,\infty}\Big) \nonumber\\
    &\lesssim\Big(\frac{2L\sqrt{d_\out}}{\epsilon}\Big)^{d_\pre K}\Big(\frac{8LKB_\infty}{\epsilon}\Big)^{d_\prg}, \label{eqn:topk_piece}
\end{align}
uniformly in $\pi\in\Pi$.
 
\textbf{Step 4: Entropy of $\hdk$.}
Combining \eqref{eqn:topk_rho}--\eqref{eqn:topk_union} and \eqref{eqn:topk_piece}, for every sample $x_1,\ldots,x_n$,
\begin{align}\label{eqn:topk_entropy}
    \log\mathcal N\big(\epsilon,\hdk,L^2(P_n)\big)
    \lesssim d_\pre K\log\frac{2L\sqrt{d_\out}}{\epsilon}+d_\prg\log\frac{8LKB_\infty}{\epsilon}+d_N\Big(\log n+2k\log\frac{eK}{k}\Big).
\end{align}
 
\textbf{Step 5: Oracle inequality.}
Taking $\alpha=\gamma$ in \eqref{eqn:offset_rad} removes the entropy integral, so
\[
    \phi(\hdk,n)\le\inf_{\gamma>0}\Big\{\frac{\log\mathcal N\big(\gamma,\hdk,L^2(P_n)\big)}{cn}+4\gamma\Big\}.
\]
Let $D:=d_\pre K+d_\prg$. By \eqref{eqn:topk_entropy}, the $\gamma$-dependent part of the right-hand side is $\frac{D\log(1/\gamma)}{cn}+4\gamma$ plus terms not depending on $\gamma$, which is minimized at $\gamma^\ast=D/(4cn)$. Substituting $\gamma^\ast$,
\begin{align*}
    \phi(\hdk,n)\lesssim{}&\frac{D}{n}+\frac{d_\pre K\log\!\big(\frac{L\sqrt{d_\out}\,n}{D}\big)+d_\prg\log\!\big(\frac{LKB_\infty\,n}{D}\big)}{n}
    +\frac{d_N\big(\log n+2k\log(eK/k)\big)}{n}.
\end{align*}
which is exactly \eqref{eqn:complexity_top-k}, completing the proof. 
\end{proof}

\subsection{Top-p MoE Class}
\subsubsection{Proof of Theorem \ref{thm:complexity_top-p}}
\begin{proof}
We use the notation of the proof of Theorem~\ref{thm:complexity_top_k}, with $\sigma_{\tilde g}$ replaced by the top-$p$ active set $S_{\tilde g}$ and
\[
    \tilde{f}(x)=\sum_{i\in S_{\tilde g}(x)}\hat g_i(x)\,a_i(x),
    \qquad
    \hat g_i(x)=\frac{g_i(x)}{\sum_{j\in S_{\tilde g}(x)}g_j(x)}.
\]
 
\textbf{Step 0: Size and mass of the active set.}
Let $g_{(K)}(x)\ge\cdots\ge g_{(1)}(x)$ denote the sorted weights. The $\kappa$ largest of $K$ weights summing to one have total at least $\kappa/K\ge p$. Hence the active set has size at most $\kappa$, i.e.\ $1\le|S_{\tilde g}(x)|\le\kappa$. Moreover, by definition of top-$p$ routing, $\sum_{j\in S_{\tilde g}(x)}g_j(x)\ge p$ for every $x$.
 
\textbf{Step 1: Partition by selection patterns.}
$\mathcal S_{\tilde{g}}$ takes values in $\{A\subseteq[K]:1\le|A|\le\kappa\}$, a set of size $\sum_{j=1}^{\kappa}\binom{K}{j}\le(eK/\kappa)^{\kappa}$. Defining $\pi$, $\Pi$, $\mathcal G_\pi$ and $\mathcal H_\pi$ as in Step~1 of the proof of Theorem~\ref{thm:complexity_top_k}, Natarajan's lemma gives
\begin{align}\label{eqn:topp_patterns}
    \log|\Pi|\le d_{N,p}\Big(\log n+2\kappa\log\frac{eK}{\kappa}\Big),
    \qquad
    \mathcal N(\epsilon,\hdp,\rho_n)\le|\Pi|\max_{\pi\in\Pi}\mathcal N(\epsilon,\mathcal H_\pi,\rho_n).
\end{align}
 
\textbf{Step 2: Perturbation bound inside one pattern.}
Fix $\pi$, $\tilde{f}=(\tilde g,\tilde a)$, $\tilde{r}=(\tilde h,\tilde b)\in\mathcal H_\pi$, and $t\le n$. Both models select the same set $A$ at $x_t$, with $|A|\le\kappa$. Exactly as in \eqref{eqn:topk_split}--\eqref{eqn:topk_weights}, with $|A|\le\kappa$ in place of $k$ and $S\ge p$ (Step~0) in place of $S\ge k/K$,
\[
    \sum_{i\in A}|\hat g_i(x_t)-\hat h_i(x_t)|\le\frac{2|A|\,\delta}{S}\le\frac{2\kappa}{p}\,\delta,
\]
and therefore
\begin{align}\label{eqn:topp_lip}
    \rho_n(f,r)\le\sqrt{d_\out}\,\|\tilde a-\tilde b\|_{\mathcal E,\infty}+\frac{2\kappa B_\infty}{p}\,\|\tilde g-\tilde h\|_{\mathcal G,\infty}
    \qquad\text{for all } f,r\in\mathcal H_\pi.
\end{align}
 
\textbf{Step 3: Covering one pattern class.}
By \eqref{eqn:topp_lip}, it suffices to take $\|\tilde a-\tilde b\|_{\mathcal E,\infty}\le\epsilon/(2\sqrt{d_\out})$ and $\|\tilde g-\tilde h\|_{\mathcal G,\infty}\le p\epsilon/(4\kappa B_\infty)$ with centers in $\mathcal G_\pi$. As before, all $K$ experts must be covered, and an internal cover of $\mathcal G_\pi$ costs at most an external cover of $\mathcal G$ at half the radius. By \eqref{eqn:topk_base},
\[
    \mathcal N(\epsilon,\mathcal H_\pi,\rho_n)\lesssim\Big(\frac{2L\sqrt{d_\out}}{\epsilon}\Big)^{d_\pre K}\Big(\frac{8L\kappa B_\infty}{p\,\epsilon}\Big)^{d_\prg}.
\]
 
\textbf{Steps 4--5: Entropy and oracle inequality.}
Combining with \eqref{eqn:topk_rho} and \eqref{eqn:topp_patterns},
\[
    \log\mathcal N\big(\epsilon,\hdp,L^2(P_n)\big)\lesssim d_\pre K\log\frac{2L\sqrt{d_\out}}{\epsilon}+d_\prg\log\frac{8L\kappa B_\infty}{p\,\epsilon}+d_{N,p}\Big(\log n+2\kappa\log\frac{eK}{\kappa}\Big).
\]
The optimization over $\gamma$ is identical to Step~5 of the proof of Theorem~\ref{thm:complexity_top_k}, again with $\gamma^\ast=D/(4cn)$ and $D=d_\pre K+d_\prg$, and yields \eqref{eqn:complexity_top-p}.
\end{proof}

\subsection{Minimax Optimality of Squared Estimation Error Rates for MoE Class}

\subsubsection{Proof of Theorem \ref{thm:lower_bound}}\label{proof:lower_bound}
\begin{proof}
Throughout, $j(x)$ denotes the group of a design point $x$, and $d=d_{\mathrm{in}}d_\out$.
 
\textbf{Step 1: A subclass common to all three routing schemes.}
For $W=(W_1,\ldots,W_K)$ with $\|W_j\|_F\le R$, let $f_W(x):=W_{j(x)}x$ on the design. We claim that $f_W\in\mathcal H$ for each of the three classes, using the gate $g^0$ of Assumption~\ref{assump:lower}(ii) and the linear experts $a_j(x)=W_jx$ of Assumption~\ref{assump:lower}(iii). On a design point $x$ with $j(x)=j$, we have $g^0_j(x)=1$ and $g^0_i(x)=0$ for $i\neq j$. Hence:
\begin{itemize}
    \item \emph{dense:} $\sum_ig^0_i(x)a_i(x)=a_j(x)$;
    \item \emph{top-$k$:} the selected set contains $j$, its total weight is $1$, so the renormalized weights are $\hat g_j(x)=1$ and $\hat g_i(x)=0$ for $i\neq j$;
    \item \emph{top-$p$:} $g^0_{(1)}(x)=1\ge p$, so the active set is $\{j\}$ and $\hat g_j(x)=1$.
\end{itemize}
In all three cases the MoE output equals $W_jx=f_W(x)$. Moreover $\|f_W(x)\|\le\|W_j\|_F\|x\|\le R\le B_\infty$, so Assumption~\ref{assump:data} holds. Since $\mathcal F_0:=\{f_W\}\subseteq\mathcal H$, it suffices to lower-bound the minimax risk over $\mathcal F_0$.
 
\textbf{Step 2: Packing.}
By the Varshamov--Gilbert bound \citep[Lemma~2.9]{tsybakov2008lower}, since $Kd\ge8$ there exist $\omega^{(0)},\ldots,\omega^{(M)}\in\{0,1\}^{Kd}$ with $M\ge2^{Kd/8}$ and pairwise Hamming distance at least $Kd/8$. Index the coordinates of $\omega$ by $(j,u,v)\in[K]\times[d_\out]\times[d_{\mathrm{in}}]$, fix $a>0$, and set
\[
\big(W^{(\ell)}_j\big)_{uv}:=a\big(2\omega^{(\ell)}_{juv}-1\big)\in\{-a,a\},\qquad f_\ell:=f_{W^{(\ell)}}.
\]
Then $\|W^{(\ell)}_j\|_F^2=a^2d$, so $f_\ell\in\mathcal F_0$ whenever $a^2d\le R^2$. For $\ell\neq\ell'$, write $\Delta_j:=W^{(\ell)}_j-W^{(\ell')}_j$. Each entry of $\Delta_j$ is $0$ or $\pm2a$, and $\sum_j\|\Delta_j\|_F^2=4a^2\rho(\omega^{(\ell)},\omega^{(\ell')})$, where $\rho$ is the Hamming distance, so
\[
\frac{a^2Kd}{2}\le\sum_{j=1}^K\|\Delta_j\|_F^2\le4a^2Kd.
\]
 
\textbf{Step 3: Separation and KL divergence.}
Since each group has $m=n/K$ points,
\[
\|f_\ell-f_{\ell'}\|_n^2=\frac1n\sum_{j=1}^K\sum_{s=1}^m\|\Delta_jx_{j,s}\|^2=\frac1K\sum_{j=1}^K\operatorname{tr}\big(\Delta_j\Sigma_j\Delta_j^\top\big).
\]
By Assumption~\ref{assump:lower}(i), $\frac{\lambda}{d_{\mathrm{in}}}\|\Delta_j\|_F^2\le\operatorname{tr}(\Delta_j\Sigma_j\Delta_j^\top)\le\frac{1}{d_{\mathrm{in}}}\|\Delta_j\|_F^2$, and therefore, by Step~2,
\begin{align}\label{eqn:lb_sep}
\frac{\lambda a^2d_\out}{2}\le\|f_\ell-f_{\ell'}\|_n^2\le4a^2d_\out .
\end{align}
Under Gaussian noise, the observations are independent with $Y_t\sim\mathcal N(f(x_t),\sigma^2I_{d_\out})$, so the Kullback--Leibler divergence between the laws $P_\ell$ and $P_{\ell'}$ of the sample is
\[
\mathrm{KL}(P_\ell\,\|\,P_{\ell'})=\frac{n}{2\sigma^2}\|f_\ell-f_{\ell'}\|_n^2\le\frac{2na^2d_\out}{\sigma^2}.
\]
 
\textbf{Step 4: Fano's method.}
Choose
\[
a^2:=\frac{\log 2}{256}\cdot\frac{\sigma^2Kd_{\mathrm{in}}}{n}.
\]
Then $\mathrm{KL}(P_\ell\|P_{0})\le\frac{\log 2}{128}\,Kd=\frac1{16}\cdot\frac{Kd}{8}\log2\le\frac1{16}\log M$ for every $\ell$. The condition $a^2d\le R^2$ holds because $n\ge\sigma^2Kd_{\mathrm{in}}d/R^2$. By \eqref{eqn:lb_sep}, the hypotheses are $2s$-separated in $\|\cdot\|_n$ with $s^2:=\lambda a^2d_\out/8$. Theorem~2.5 of \citet{tsybakov2008lower}, applied with $\alpha=1/16$ and $M\ge2$, gives
\[
\inf_{\hat f}\max_{0\le\ell\le M}P_\ell\big(\|\hat f-f_\ell\|_n\ge s\big)\ge\frac{\sqrt M}{1+\sqrt M}\Big(1-\frac18-\sqrt{\frac{1}{8\log M}}\Big)\ge\frac14 .
\]
By Markov's inequality, $\mathbb E_\ell\|\hat f-f_\ell\|_n^2\ge s^2P_\ell(\|\hat f-f_\ell\|_n\ge s)$, so
\[
\inf_{\hat f}\sup_{f\in\mathcal H}\mathbb E\|\hat f-f\|_n^2\ge\inf_{\hat f}\max_{\ell}\mathbb E_\ell\|\hat f-f_\ell\|_n^2\ge\frac{s^2}{4}=\frac{\lambda a^2d_\out}{32}=\frac{\log2}{8192}\cdot\lambda\,\frac{\sigma^2Kd_{\mathrm{in}}d_\out}{n},
\]
which is \eqref{eqn:lower_bound} with $c=\log2/8192$.
\end{proof}

\newpage
\section{Supplementary To \emph{Experiments }}

\subsection{Detailed Experiments on the Theoretical Upper Bounds obtained in Theorems ~\ref{thm:complexity_general}, ~\ref{thm:complexity_top_k} \& ~\ref{thm:complexity_top-p}}\label{app:experiments}

\subsubsection{General MoE Experiments: Details across All Configurations}\label{app:general_experiments}
\paragraph{Excess risk.}
Since The theorem~\ref{thm:complexity_general} bounds the estimation error rather than the total test error, so we remove the irreducible part (noise and approximation error) before plotting. We fit $\mathrm{MSE}(n)=a+b\,n^{-\beta}$ to the test MSE by profile least squares: for each candidate floor $a$ below the smallest observed MSE, $(b,\beta)$ come from a straight-line fit in log-log scale, and we keep the $a$ with the smallest squared residual. We plot the excess test MSE $\mathrm{MSE}(n)-\hat a$ and report the slope of its least-squares line in log-log scale.

\paragraph{Bound shape.}
The dashed curves show
\[
B(n)=\frac{D+Kd_\pre\log\!\big(L\sqrt{d_{\mathrm{out}}}\,n/D\big)+d_\prg\log\!\big(LK B_\infty n/D\big)}{n},
\]
which is the bound of Theorem~\ref{thm:complexity_general}. Unlike the Top-$k$ case, there is no routing term, since all $K$ experts are active. We set $L=1$ and take $B_\infty$ to be the largest recorded expert norm, which enters only inside a logarithm. Because the bound holds only up to absolute constants, $B(n)$ is rescaled to pass through the first data point, so only its slope is meaningful. Over a finite range, the local log-log slope of $\log(Cn/D)/n$ is $-1+1/\log(Cn/D)$, so the bound shape is flatter than $-1$ for moderate $n/D$ and approaches $-1$ only as $n/D\to\infty$.

\paragraph{Results.}
Figure~\ref{fig:general_loglog_full} shows all twelve configurations, and Table~\ref{tab:general_slopes} compares the fitted slopes with the slope of $B(n)$ over the same range of $n$. The bound shape has slope between $-0.80$ and $-0.84$ in every configuration. The empirical slopes lie between $-0.82$ and $-1.42$, with mean $-1.00$ and standard deviation $0.15$. In every configuration the empirical decay is at least as steep as that of $B(n)$ (up to the borderline case $(32,8)$), which is consistent with Theorem~\ref{thm:complexity_general} being an upper bound: the measured excess risk decays at least as fast as the bound predicts. Ten of the twelve fitted slopes are within $0.1$ of the parametric slope $-1$.

Two configurations deviate from $-1$:
\begin{enumerate}
    \item For $(d_{\mathrm{in}},d_{\mathrm{out}})=(4,32)$ the slope is $-1.42$, steeper than the parametric rate. We attribute this to the floor estimate: if $\hat a$ is slightly overestimated, the excess MSE is biased downward at large $n$, which steepens the fitted slope. The floor and the exponent $\beta$ are only weakly identified jointly, and this is the configuration with the smallest excess MSE ($\approx 2\times10^{-5}$ at $n_{\max}$).
    \item For $(16,32)$ and $(32,8)$ the slopes are $-0.83$ and $-0.82$, flatter than $-1$. As in the Top-$k$ experiments, this is plausibly due to a nonvanishing optimization error from a fixed number of AdamW epochs and to the fact that the target is a Transformer, not an MoE, so the best-in-class floor is only approximately reached.
\end{enumerate}
At the largest sample sizes the excess MSE is of order $10^{-4}$--$10^{-3}$, comparable to the fluctuations of the test MSE, which explains the scatter at the right end of several panels.

\begin{table}[h]
    \centering
    \caption{General MoE ($K=4$, dense routing): parameter count $D$, the ratio $n_{\max}/D$ with $n_{\max}=8\times10^5$, the fitted log-log slope of the excess test MSE, and the slope of the bound shape $B(n)$ over the same range of $n$.}
    \label{tab:general_slopes}
    \begin{tabular}{cccccc}
        \hline
        $d_{\mathrm{in}}$ & $d_{\mathrm{out}}$ & $D$ & $n_{\max}/D$ & empirical slope & bound-shape slope \\
        \hline
        4  & 8  & 1988 & 402 & $-1.03$ & $-0.84$ \\
        4  & 16 & 3044 & 263 & $-1.07$ & $-0.83$ \\
        4  & 32 & 5156 & 155 & $-1.42$ & $-0.82$ \\
        8  & 8  & 2628 & 304 & $-0.99$ & $-0.84$ \\
        8  & 16 & 3684 & 217 & $-0.90$ & $-0.83$ \\
        8  & 32 & 5796 & 138 & $-0.96$ & $-0.81$ \\
        16 & 8  & 3908 & 205 & $-0.99$ & $-0.82$ \\
        16 & 16 & 4964 & 161 & $-0.98$ & $-0.82$ \\
        16 & 32 & 7076 & 113 & $-0.83$ & $-0.81$ \\
        32 & 8  & 6468 & 124 & $-0.82$ & $-0.81$ \\
        32 & 16 & 7524 & 106 & $-0.98$ & $-0.81$ \\
        32 & 32 & 9636 & 83 & $-0.99$ & $-0.80$ \\
        \hline
        \multicolumn{4}{r}{mean (std.)} & $-1.00$ $(0.15)$ & $-0.82$ \\
        \hline
    \end{tabular}
\end{table}

\begin{figure}
    \centering
    \includegraphics[width=0.8\linewidth]{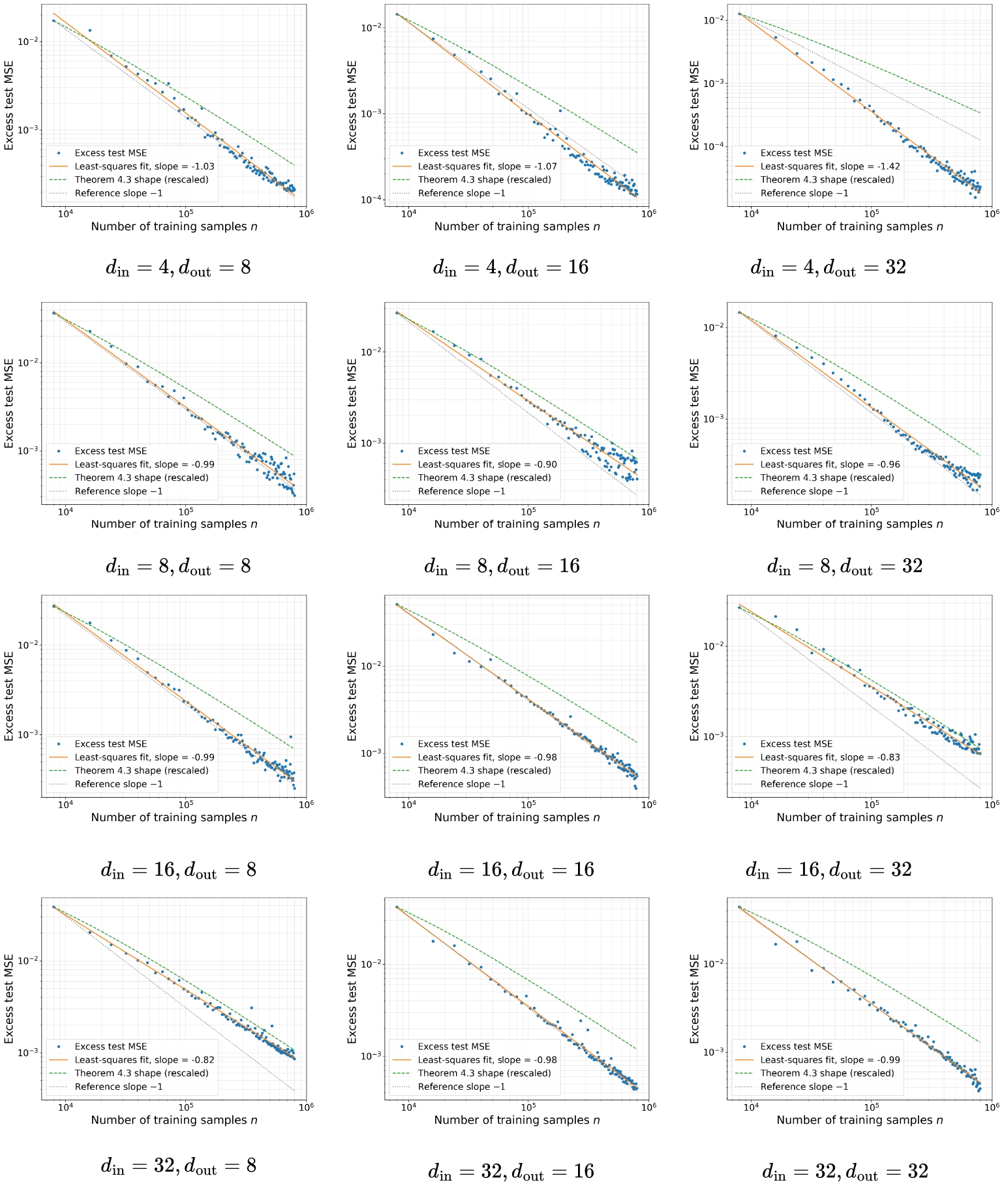}
    \caption{Excess test MSE versus $n$ (log-log) for the general (dense) MoE ($K=4$) in all twelve configurations. Points: excess test MSE $\mathrm{MSE}(n)-\hat a$; solid: least-squares fit (slope in legend); dashed: bound shape $B(n)$ of Theorem~\ref{thm:complexity_general}, rescaled to the first point; dotted: reference slope $-1$.}
    \label{fig:general_loglog_full}
\end{figure}


\newpage
\subsubsection{Top-$k$ Experiments: Details across All Configurations}\label{app:topk_experiments}
\paragraph{Excess risk.}
The Theorem~\ref{thm:complexity_top_k} bound the estimation error, not the total test error, so we remove the irreducible part (noise and approximation error) before plotting. We fit $\mathrm{MSE}(n)=a+b\,n^{-\beta}$ to the test MSE by profile least squares. For each candidate floor $a$ below the smallest observed MSE, $(b,\beta)$ are obtained from a straight-line fit in log--log scale, and we keep the $a$ with the smallest squared residual. We then plot the excess test MSE, $\mathrm{MSE}(n)-\hat a$, and report the slope of its least-squares line in log--log scale. Squared errors are used throughout, matching the squared loss in Theorem~\ref{thm:complexity_top_k}; averaging over the $d_{\mathrm{out}}$ output coordinates changes the MSE only by a constant factor, which does not affect slopes.
 
\paragraph{Bound shape.}
The dashed curves show
\[
B(n)=\frac{D+Kd_\pre\log\!\big(L\sqrt{d_{\mathrm{out}}}\,n/D\big)+d_\prg\log\!\big(LKB_\infty n/D\big)}{n},
\]
which is the bound of Theorem~\ref{thm:complexity_top_k} without its routing term $d_N(\log n+2k\log(eK/k))/n$. The routing term is omitted because the Natarajan dimension of a GELU gate is not computable. We set $L=1$ and take $B_\infty$ to be the largest recorded expert norm; $B_\infty$ enters only inside a logarithm, and varying it between $1$ and $20$ changes the slopes in Table~\ref{tab:topk_slopes} by at most $0.02$. Because the bound holds only up to absolute constants, $B(n)$ is rescaled to pass through the first data point, so only its slope is meaningful. Over a finite range, the local log--log slope of $\log(Cn/D)/n$ is $-1+1/\log(Cn/D)$; the bound shape is therefore flatter than $-1$ for moderate $n/D$, and approaches $-1$ only as $n/D\to\infty$.
 
\paragraph{Results.}
Figure~\ref{fig:top_k_loglog_full} shows all twelve configurations, and Table~\ref{tab:topk_slopes} compares the fitted slopes with the slope of $B(n)$ over the same range of $n$. The bound shape has slope between $-0.81$ and $-0.85$ in every configuration. The empirical slopes lie between $-0.62$ and $-1.09$, with mean $-0.83$ and standard deviation $0.13$, so on average the empirical and theoretical rates agree closely. In six of the twelve configurations the empirical slope is steeper than that of $B(n)$. This is consistent with Theorem~\ref{thm:complexity_top_k}, which is an upper bound and need not be tight for a particular target. In the remaining configurations the decay is flatter, most visibly for $(16,16)$ and $(32,8)$. Over a finite range of $n$ this does not contradict an upper bound that holds up to constants, and we attribute it to three effects that the theory does not model:
\begin{enumerate}
    \item The floor $\hat{a}$ and the exponent $\beta$ are only weakly identified jointly. A slightly underestimated floor leaves a constant offset in the excess MSE, which flattens the fitted slope at large $n$.
    \item The networks are trained by a fixed number of AdamW epochs rather than solved to global optimality, so the measured error includes an optimization error that need not decay as $1/n$.
    \item The target is not an MoE, so the best-in-class predictor, and hence the floor, is itself only approximately reached.
\end{enumerate}
At the largest sample sizes the excess MSE is of order $10^{-4}$--$10^{-3}$, comparable to the fluctuations of the test MSE, which explains the increased scatter at the right end of several panels.
 
\begin{table}[h]
    \centering
    \caption{Top-$k$ MoE ($K=4$, $k=2$): parameter count $D$, the ratio $n_{\max}/D$ with $n_{\max}=8\times10^5$, the fitted log--log slope of the excess test MSE, and the slope of the bound shape $B(n)$ over the same range of $n$.}
    \label{tab:topk_slopes}
    \begin{tabular}{cccccc}
        \hline
        $d_{\mathrm{in}}$ & $d_{\mathrm{out}}$ & $D$ & $n_{\max}/D$ & empirical slope & bound-shape slope \\
        \hline
        4  & 8  & 1988 & 402 & $-0.88$ & $-0.85$ \\
        4  & 16 & 3044 & 263 & $-0.92$ & $-0.84$ \\
        4  & 32 & 5156 & 155 & $-1.09$ & $-0.83$ \\
        8  & 8  & 2628 & 304 & $-0.95$ & $-0.84$ \\
        8  & 16 & 3684 & 217 & $-0.81$ & $-0.84$ \\
        8  & 32 & 5796 & 138 & $-0.76$ & $-0.83$ \\
        16 & 8  & 3908 & 205 & $-0.84$ & $-0.83$ \\
        16 & 16 & 4964 & 161 & $-0.67$ & $-0.83$ \\
        16 & 32 & 7076 & 113 & $-0.94$ & $-0.82$ \\
        32 & 8  & 6468 & 124 & $-0.62$ & $-0.81$ \\
        32 & 16 & 7524 & 106 & $-0.78$ & $-0.82$ \\
        32 & 32 & 9636 & 83  & $-0.74$ & $-0.82$ \\
        \hline
        \multicolumn{4}{r}{mean (std.)} & $-0.83$ $(0.13)$ & $-0.83$ \\
        \hline
    \end{tabular}
\end{table}

\begin{figure}
    \centering
    \includegraphics[width=0.8\linewidth]{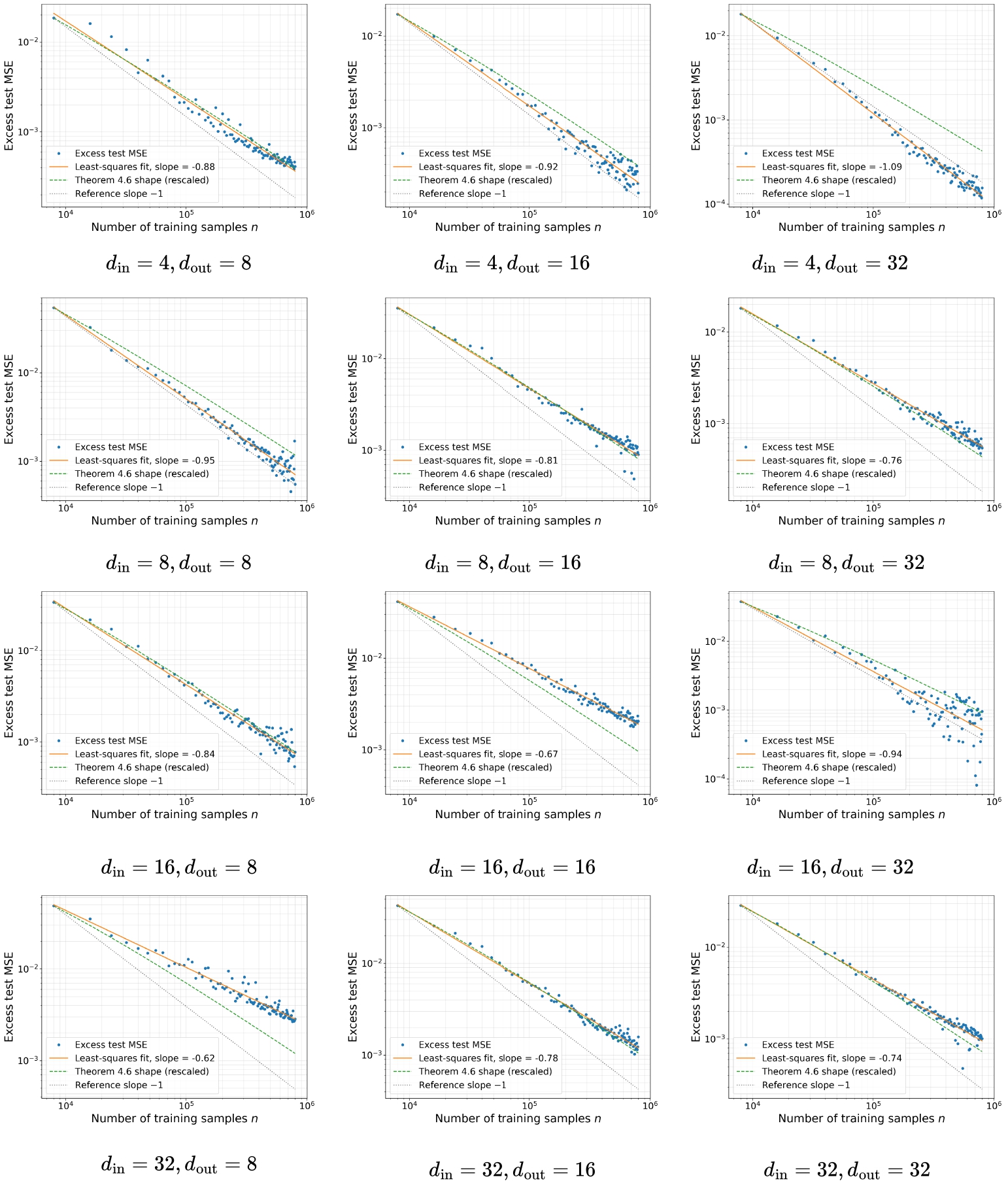}
    \caption{Excess test MSE versus $n$ (log--log) for the Top-$k$ MoE ($K=4$, $k=2$) in all twelve configurations $(d_{\mathrm{in}},d_{\mathrm{out}})\in\{4,8,16,32\}\times\{8,16,32\}$. Points: excess test MSE $\mathrm{MSE}(n)-\hat a$; solid: least-squares fit (slope in legend); dashed: bound shape $B(n)$ of Theorem~\ref{thm:complexity_top_k}, rescaled to the first point; dotted: reference slope $-1$.}
    \label{fig:top_k_loglog_full}
\end{figure}

\newpage 
\subsubsection{Top-$p$ Experiments: Details across All Configurations}\label{app:topp_experiments}

As before, our MoE has $K=4$ experts and Top-$p$ routing with threshold $p=0.7$: at each input, the smallest set of largest gating weights whose cumulative mass reaches $p$ is activated, and the selected weights are renormalized. By Step~0 of the proof of Theorem~\ref{thm:complexity_top-p}, at most $\kappa=\lceil Kp\rceil=3$ experts are active at any input. Each expert and the gate have one hidden layer of width $d_h=32$ with GELU activations. 
The parameter counts $d_\pre$, $d_\prg$ and $D=Kd_\pre+d_\prg$ are computed as in Table~\ref{tab:topp_slopes} with $d_h=32$.
 
\paragraph{Bound shape.}
The dashed curves show the bound of Theorem~\ref{thm:complexity_top-p} without its routing term $d_{N,p}(\log n+2\kappa\log(eK/\kappa))/n$:
\[
B_p(n)=\frac{D+Kd_\pre\log\!\big(L\sqrt{d_{\mathrm{out}}}\,n/D\big)+d_\prg\log\!\big(L\kappa B_\infty n/(pD)\big)}{n}.
\]
As in the Top-k experiments, we set $L=1$, take $B_\infty$ to be the largest recorded expert norm, and rescale $B_p(n)$ to pass through the first data point, so as to make its slope meaningful. Varying $B_\infty$ between $1$ and $20$ changes the slopes in Table~\ref{tab:topp_slopes} by at most $0.03$. Compared with the Top-$k$ bound, the threshold $p$ enters only through the factor $\kappa/p\le K+1/p$ inside the gating logarithm, and through the routing term.

\paragraph{Results.}
Figure~\ref{fig:top-p_loglog_full} shows all twelve configurations, and Table~\ref{tab:topp_slopes} compares the fitted slopes with the slope of $B_p(n)$ over the same range of $n$. The bound shape has a slope between $-0.78$ and $-0.83$. The empirical slopes lie between $-0.53$ and $-1.10$, with mean $-0.81$ and standard deviation $0.14$. As for Top-$k$, the empirical and theoretical rates therefore agree on average:
\begin{itemize}
    \item in six configurations the empirical decay is steeper than $B_p(n)$, and in $(4,8)$ the two slopes coincide, which is consistent with Theorem~\ref{thm:complexity_top-p} being an upper bound;
    \item in the remaining five the decay is flatter.
\end{itemize}
The clearest flat case is $(d_{\mathrm{in}},d_{\mathrm{out}})=(32,8)$, with slope $-0.53$. This configuration is also the flattest in the Top-$k$ experiments (slope $-0.62$, Table~\ref{tab:topk_slopes}). Since it is the same configuration under both routing rules, the deviation is most likely due to the target and to optimization rather than to the routing mechanism. The same three caveats apply as in Appendix~\ref{app:topk_experiments}: the floor and the exponent are only weakly identified jointly, the networks are trained by a fixed number of AdamW epochs rather than solved to optimality, and the Transformer target is not an MoE.
 
 
\begin{table}[h]
    \centering
    \caption{Top-$p$ MoE ($K=4$, $p=0.7$, $\kappa=3$, $d_h=32$): parameter count $D$, the ratio $n_{\max}/D$ with $n_{\max}=8\times10^5$, the fitted log--log slope of the excess test MSE, and the slope of the bound shape $B_p(n)$ over the same range of $n$.}
    \label{tab:topp_slopes}
    \begin{tabular}{cccccc}
        \hline
        $d_{\mathrm{in}}$ & $d_{\mathrm{out}}$ & $D$ & $n_{\max}/D$ & empirical slope & bound-shape slope \\
        \hline
        4  & 8  & 1988  & 402 & $-0.83$ & $-0.83$ \\
        4  & 16 & 3044  & 263 & $-0.97$ & $-0.82$ \\
        4  & 32 & 5156 & 155  & $-1.10$ & $-0.81$ \\
        8  & 8  & 2628  & 304 & $-0.79$ & $-0.82$ \\
        8  & 16 & 3684  & 217 & $-0.74$ & $-0.81$ \\
        8  & 32 & 5796 & 138  & $-0.76$ & $-0.81$ \\
        16 & 8  & 3908  & 205 & $-0.84$ & $-0.80$ \\
        16 & 16 & 4964  & 161  & $-0.82$ & $-0.80$ \\
        16 & 32 & 7076 & 113  & $-0.65$ & $-0.80$ \\
        32 & 8  & 6468 & 124  & $-0.53$ & $-0.78$ \\
        32 & 16 & 7524 & 106  & $-0.84$ & $-0.79$ \\
        32 & 32 & 9636 & 83  & $-0.87$ & $-0.79$ \\
        \hline
        \multicolumn{4}{r}{mean (std.)} & $-0.81$ $(0.14)$ & $-0.80$ \\
        \hline
    \end{tabular}
\end{table}

\begin{figure}
    \centering
    \includegraphics[width=0.8\linewidth]{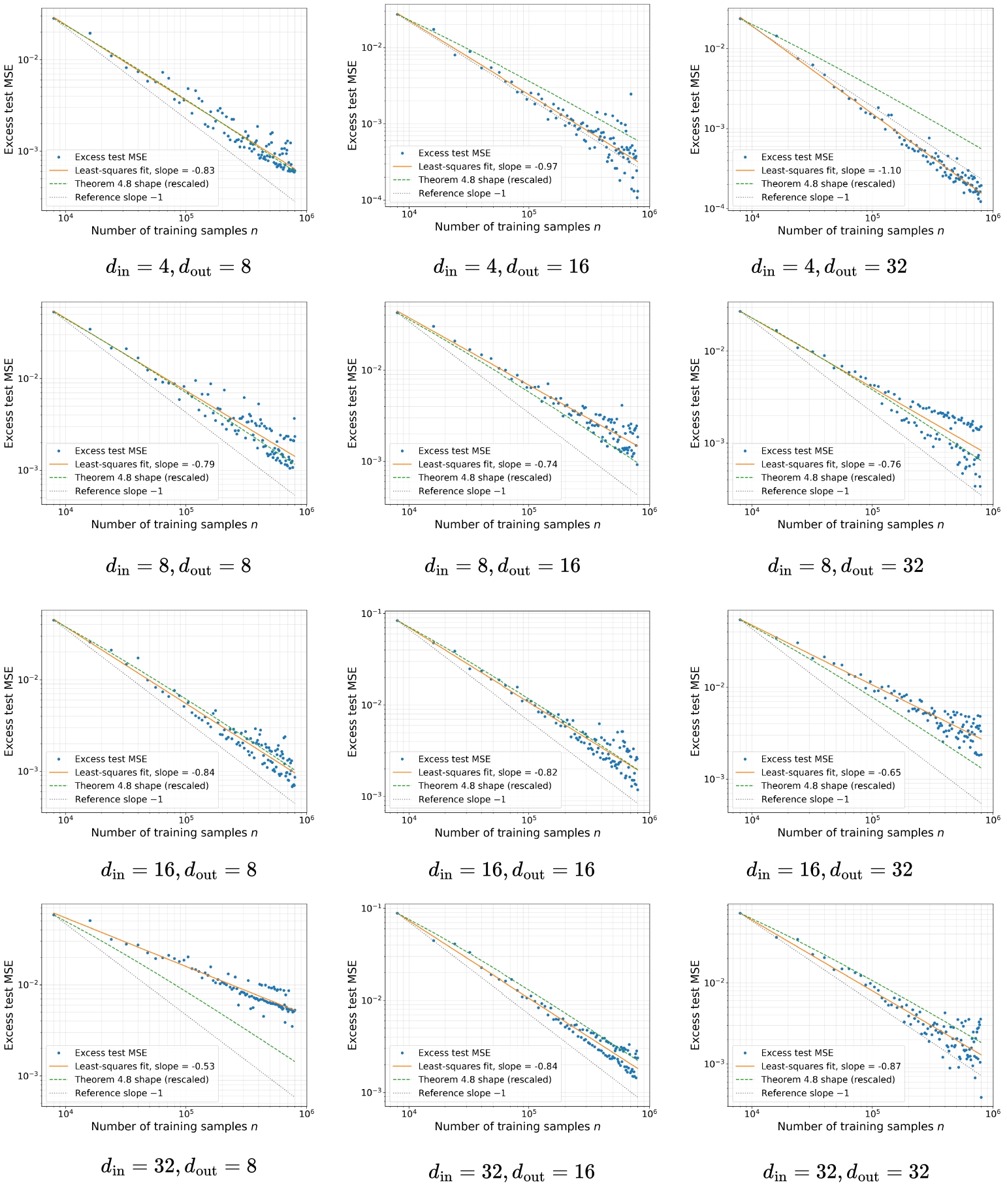}
    \caption{Excess test MSE versus $n$ (log--log) for the Top-$p$ MoE ($K=4$, $p=0.7$) in all twelve configurations $(d_{\mathrm{in}},d_{\mathrm{out}})\in\{4,8,16,32\}\times\{8,16,32\}$. Points: excess test MSE $\mathrm{MSE}(n)-\hat a$; solid: least-squares fit (slope in legend); dashed: bound shape $B_p(n)$ of Theorem~\ref{thm:complexity_top-p}, rescaled to the first point; dotted: reference slope $-1$.}
    \label{fig:top-p_loglog_full}
\end{figure}

\newpage 

\subsubsection{Top-$p$ Routing: Average Number of Expert Selections across All Combinations}\label{app:top-p_shannon_full}

\paragraph{Quantity.}
For each training size $n$, the Top-$p$ MoE ($K=4$, $p=0.7$) is evaluated on the test set. For a test input $x$, the gate probabilities are sorted in decreasing order and $k_x$ is the smallest number of experts whose cumulative probability reaches $p$. Figure~\ref{fig:top-p_shannon_full} shows the test-set average $\bar k_x$ against $n$ for all twelve configurations. The dotted line marks hard routing ($k_x=1$) and the dashed line marks $\lceil Kp\rceil=3$, the worst-case number of active experts used in Theorem~\ref{thm:complexity_top-p}. Since $k_x\ge1$ for every input, the fraction of test inputs with $k_x\ge2$ is at most $\bar k_x-1$.

\paragraph{Results.}
In all twelve configurations $\bar k_x$ decreases with $n$, mostly before $n\approx10^{5}$, and then stays roughly flat. At the smallest training size ($n=8\times10^{3}$) it lies between about $1.2$ and $1.6$, and at $n=8\times10^{5}$ it lies between $1.03$ and $1.17$ (Table~\ref{tab:shannon_entropy_top-p}). Hence at the largest sample size at most $3\%$ to $17\%$ of test inputs use more than one expert, depending on the configuration. The curves never come close to the worst-case count $\lceil Kp\rceil=3$, so this count is conservative for trained routers. The final values tend to be larger for larger $d_{\mathrm{out}}$ and $d_{\mathrm{in}}$ (about $1.03$ for $(4,8)$ and $1.17$ for $(16,32)$), but the ordering is not strict. The curves are also noisier at large $n$ for $d_{\mathrm{out}}=32$, where $\bar k_x$ fluctuates by about $0.05$ between neighboring training sizes. The entries of Table~\ref{tab:shannon_entropy_top-p} are single-point values at $n=8\times10^{5}$ and share this fluctuation.

\paragraph{Interpretation.}
In all configurations the learned Top-$p$ router is close to hard routing, so the rate in $n$ reported in Section~\ref{sec:experiments} is measured in a regime where multi-expert routing is used on a small fraction of inputs. These runs therefore do not probe the dependence of the bound on $p$ or on the number of active experts. We report the decrease of $\bar k_x$ as an observation and do not attribute it to a specific cause. It may reflect the gate becoming sharper with more data, or optimization dynamics under a fixed number of epochs. Each curve comes from a single training run per training size.

\begin{figure}[h]
    \centering
    \includegraphics[width=0.8\linewidth]{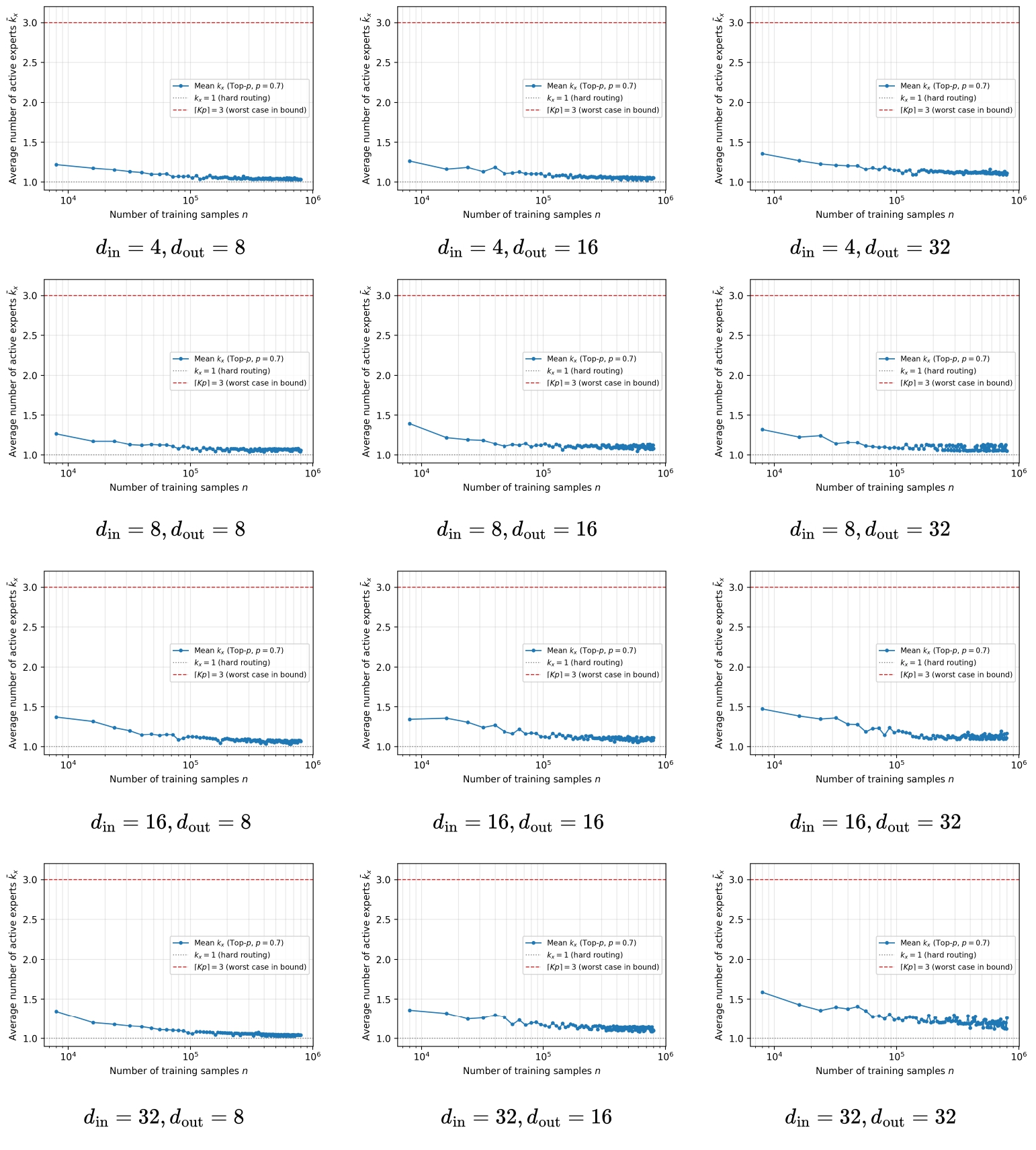}
    \caption{Average number of experts $\bar k_x$ selected by Top-$p$ routing ($K=4$, $p=0.7$) on the test set versus training size $n$, for all twelve configurations $(d_{\mathrm{in}},d_{\mathrm{out}})\in\{4,8,16,32\}\times\{8,16,32\}$. Dotted: hard routing ($k_x=1$); dashed: worst-case count $\lceil Kp\rceil=3$ used in the bound.}
    \label{fig:top-p_shannon_full}
\end{figure}














\end{document}